\documentclass[journal]{IEEEtran}
\usepackage{amsmath,amsfonts}
\usepackage{algorithmic}
\usepackage{algorithm}
\usepackage{array}
\usepackage[caption=false,font=normalsize,labelfont=sf,textfont=sf]{subfig}
\usepackage{textcomp}
\usepackage{stfloats}
\usepackage{url}
\usepackage{verbatim}
\usepackage{graphicx}
\usepackage{cite}
\usepackage{multirow}
\usepackage{amsthm}
\usepackage{booktabs}
\usepackage{multirow}
\usepackage[table]{xcolor}
\newtheorem{theorem}{Theorem}
\newtheorem{lemma}{Lemma}
\newtheorem{corollary}{Corollary}
\newtheorem{remark}{Remark}
\newtheorem{assumption}{Assumption} 
\begin{document}

\title{ADP-VRSGP: Decentralized Learning with Adaptive Differential Privacy via Variance-Reduced Stochastic Gradient Push}

\author{Xiaoming Wu,Teng Liu,
Xin Wang,~\IEEEmembership{Member,~IEEE,}
Ming Yang,~\IEEEmembership{Member,~IEEE,} and 
Jiguo Yu,~\IEEEmembership{Fellow,~IEEE}

\thanks{This work was supported in part by the Taishan Scholars Program under Grants tsqn202211203 and tsqn202408239, in part by the NSFC under Grants 62402256 and 62272256, in part by the Shandong Provincial Nature Science Foundation of China under Grants ZR2024MF100 and ZR2022ZD03, and in part by the Pilot Project for the Integration of Science, Education and Industry of Qilu University of Technology (Shandong Academy of Sciences) under Grant 2025ZDZX01 (Corresponding author: Xin Wang).}
\thanks{X. Wu, T. Liu, X. Wang, and M. Yang are with the Key Laboratory of Computing Power Network and Information Security, Ministry of Education, Shandong Computer Science Center, Qilu University of Technology (Shandong Academy of Sciences), Jinan 250014, China, and also with Shandong Provincial Key Laboratory of Industrial Network and Information System Security, Shandong Fundamental Research Center for Computer Science, Jinan 250014, China. Emails: {\small wuxm@sdas.org, 10431230028@stu.qlu.edu.cn, xinwang@qlu.edu.cn, yangm@sdas.org}.}
\thanks{J. Yu is with the School of Information and Software Engineering, University of Electronic Science and Technology of China, Chengdu 611731, China, and also with the Big Data Institute, Qilu University of Technology, Jinan 250353, China, Email: jiguoyu@sina.com.}}

\markboth{IEEE Internet of Things Journal,~Vol.~xx, No.~xx, August~2025}%
{Shell \MakeLowercase{\textit{et al.}}: A Sample Article Using IEEEtran.cls for IEEE Journals}

\IEEEpubid{0000--0000/00\$00.00~\copyright~2021 IEEE}

\maketitle

\begin{abstract}
Differential privacy is widely employed in decentralized learning to safeguard sensitive data by introducing noise into model updates. However, existing approaches that use fixed-variance noise often degrade model performance and reduce training efficiency. 
To address these limitations, we propose a novel approach called decentralized learning with adaptive differential privacy via variance-reduced stochastic gradient push (ADP-VRSGP). This method dynamically adjusts both the noise variance and the learning rate using a stepwise-decaying schedule, which accelerates training and enhances final model performance while providing node-level personalized privacy guarantees. To counteract the slowed convergence caused by large-variance noise in early iterations, we introduce a progressive gradient fusion strategy that leverages historical gradients. Furthermore, ADP-VRSGP incorporates decentralized push-sum and aggregation techniques, making it particularly suitable for time-varying communication topologies. Through rigorous theoretical analysis, we demonstrate that ADP-VRSGP achieves robust convergence with an appropriate learning rate, significantly improving training stability and speed. Experimental results validate that our method outperforms existing baselines across multiple scenarios, highlighting its efficacy in addressing the challenges of privacy-preserving decentralized learning.
\end{abstract}


\section{Introduction}
\IEEEPARstart{I}{n} the era of Artificial Intelligence of Things (AIoT, AI+IoT), machine learning has emerged as a transformative tool for addressing complex problems, particularly through the utilization of large-scale IoT data to train intelligent models~\cite{10734328}. Meanwhile, the rapid development of sensing and communication technologies has accelerated the real-time generation and transmission of massive IoT data, intensifying the demand for advanced computational resources to efficiently handle these data utilization tasks. In this context, decentralized machine learning has gained prominence as a promising edge intelligence paradigm, enabling multiple computational nodes to collaboratively train models and offering an effective solution for large-scale IoT data learning tasks~\cite{10559884}. 

In recent years, researchers have increasingly turned their attention to fully decentralized learning methods that operate under directed and time-varying communication topologies, where each node communicates exclusively with its neighbors without relying on a central node~\cite{10988816}. Traditional distributed methods often depend on a central node for parameter aggregation, which can lead to significant communication overhead and introduce the risk of a single point of failure \cite{lian2017can}. Furthermore, in many real-world scenarios, the communication topology between computational nodes may not adhere to the common assumption of being undirected and time-varying, as is typically presumed in many fully distributed methods~\cite{assran2019stochastic}. To address this limitation, this paper adopts the stochastic gradient push (SGP) scheme~\cite{10846935}, which removes the need for a central node while effectively addressing the challenges posed by directed and time-varying communication topologies, thereby enabling robust and efficient decentralized learning.

\IEEEpubidadjcol

However, decentralized learning approaches remain vulnerable to privacy breaches, necessitating careful consideration of privacy-preserving mechanisms in algorithm design. Although computational nodes in decentralized learning typically maintain data isolation and collaborate by sharing only model parameters or gradients \cite{9348921}, sensitive information from the original IoT data may still be inferred by adversaries through sophisticated model inversion attacks based on the messages exchanged between nodes~\cite{bai2024membership,wang2021dynamic}. 
To address this issue, differential privacy (DP)—a well-established privacy-preserving framework—has been integrated into many decentralized methods. Among DP techniques, the Gaussian mechanism is widely adopted, facilitating a differentially private stochastic gradient descent (DP-SGD) operation \cite{abadi2016deep,10163938}. This approach typically perturbs the gradients or other intermediate parameters by adding noise proportional to the upper bound or clipping threshold of the gradients. 

The noise introduced during the iterative update process can significantly degrade model performance, highlighting the need for advanced algorithmic designs to mitigate its impact. 
As the number of iterations increases, model parameters gradually converge toward their optimal values, and gradient magnitudes approach zero. During this phase, the injection of high-variance noise can severely disrupt gradient directions, further exacerbating the adverse effects of noise on accuracy. 

Researchers have proposed various strategies for adjusting noise levels throughout the training process to address this issue. For instance, Yu et al.~\cite{yu2019differentially} introduced a dynamic privacy budget allocator for training neural networks under centralized DP protection, where noise levels are adjusted according to dynamic privacy budgets. Similarly, the work~\cite{10817493} proposed a denoising DP method that removes previously added noise from the global model, enabling direct denoising at the node level. The study~\cite{9691277} proposed an iterative adaptive privacy budget allocation method that dynamically adjusts privacy budgets for different perturbation mechanisms to minimize the total bias of perturbed data. He et al.~\cite{10210010} employed strategies such as adaptive clipping, weight compression, and parameter shuffling to mitigate the impact of noise under the local DP mechanism. Wang et al.~\cite{10091486} utilized a filtering and screening method based on the exponential mechanism to reduce noise levels introduced during training. These methods not only enhance model convergence but also ensure a level of privacy protection equivalent to that of fixed-noise schemes, achieved through careful DP budget tracking under the same total privacy budget constraint~\cite{mironov2019r}. However, they do not fully exploit the potential benefits of dynamically adjusting the learning rate to reduce noise variance. 

Although widely used techniques, such as adaptive gradient methods and decaying step-size strategies, have demonstrated empirical success in adjusting learning rates \cite{NEURIPS2024_ca98452d,9531335}, designing an optimal learning rate to minimize noise-induced errors remains an underexplored challenge. This challenge is further compounded by the dependence of the learning rate on the unknown Lipschitz constant of the gradient \cite{pmlr-v202-defazio23a}, making it difficult to adapt the learning rate to varying noise levels and gradient characteristics. Furthermore, there is a lack of comprehensive theoretical analysis to validate the effectiveness of these methods in practical scenarios, underscoring a critical gap in the current literature. In addition, noise-decaying strategies often require the introduction of large noise variances during the early stages of training, which can significantly impede convergence speed. Compared to standard DP-SGD methods, this initial noise setting not only degrades convergence performance but may also lead to divergence. 

To address the challenges outlined above, we propose a novel differentially private decentralized learning algorithm incorporating two innovative modules. First, we introduce a stepwise noise-decaying mechanism combined with an adaptive learning rate strategy, designed to mitigate the adverse effects of noise on convergence performance. By dynamically adjusting the learning rate, we reduce noise variance throughout the training process, thereby achieving a tighter bound on noise-induced errors. Second, we develop a progressive gradient fusion (PGF) strategy, which leverages the accumulation of historical gradients to counteract the slowdown in convergence speed caused by large noise variances in the early stages of training. Together, these strategies enhance the stability and efficiency of the training process while maintaining robust privacy guarantees. 

The main contributions are summarized as follows:
\begin{enumerate}
	\item We propose a decentralized learning algorithm with adaptive differential privacy via variance-reduced stochastic gradient push, termed ADP-VRSGP, tailored for non-convex machine learning problems in directed and time-varying network topologies. 
	Specifically, we introduce a stepwise noise decaying mechanism combined with varying learning rate, which dynamically coordinates the learning rate and noise intensity to enhance convergence performance. Furthermore, we design a PGF method based on historical gradients, which progressively aggregates these gradients to effectively reduce noise variance and improve overall performance.
	
	\item We analyze the convergence performance of the proposed algorithm under directed, sparse, and time-varying communication topologies. Our analysis demonstrates that ADP-VRSGP achieves improved noise bounds of \( O\left(\frac{(\log T)^2}{\sqrt{nT}}\right) \) and \( O\left(\frac{1}{\sqrt{n}T^{\frac{1}{2} + 2p}} \right) \), where \( p \in \left(-\frac{1}{2}, 0\right) \). Moreover, we identify an inverse relationship between noise intensity and historical gradient staleness, and show that employing smaller gradient clipping values effectively counteracts this negative effect. 
	
	\item We conduct a series of experiments to validate the theoretical findings. The results demonstrate that our method outperforms existing algorithms in terms of test accuracy and effectively mitigates the issue of deteriorating convergence performance caused by excessive noise added in the early training stage.
	
\end{enumerate}

The remainder of this paper is organized as follows: Section~\ref{Section:2} reviews the related work, while Section~\ref{Section:3} provides the necessary background knowledge and formulates the considered problem. Section~\ref{Section:4} details the proposed algorithm, followed by a comprehensive convergence analysis in Section~\ref{Section:5}. Section~\ref{Section:6} presents and discusses the experimental results. Finally, Section~\ref{Section:7} concludes the paper and highlights potential directions for future research.

\section{Related work}\label{Section:2}
\textbf{Distributed machine learning}. With continuous advancements in hardware and communication technologies, distributed machine learning has emerged as a widely studied research area. The communication topology in such systems is generally classified into two categories: with or without a central node. The first category relies on a central server to coordinate the training process \cite{10400772}, while the fully decentralized architecture allows nodes to communicate directly with one another, enhancing system flexibility and fault tolerance\cite{10457953}. For instance, Lalitha et al.~\cite{lalitha2018fully} proposed a fully decentralized framework where users update model parameters by aggregating information from neighboring nodes. Similarly, the decentralized SGD algorithm introduced by \cite{NIPS2017_f7552665} reduces communication costs and demonstrates greater efficiency than centralized communication topologies in certain scenarios. Nevertheless, these methods are primarily designed for static networks and lack adaptability to dynamic communication environments. In more general environments, where the communication structure between nodes is directed and time-varying, these methods lack the adaptability and flexibility required to accommodate diverse communication patterns. Furthermore, differences in communication architectures pose challenges to directly applying existing DP algorithms to time-varying communication networks, increasing the complexity of convergence analysis in such systems.

\textbf{Privacy-preserving methods}. Various techniques have emerged to protect user privacy, including homomorphic encryption, secure multi-party computation, and differential privacy. Among these, DP, with its rigorous mathematical foundations, has been widely adopted in machine learning and is regarded as an effective method for safeguarding data privacy~\cite{dwork2006our,wang2020privacy}. The work~\cite{abadi2016deep} was the first to apply DP to gradient descent methods in deep learning, introducing the DP-SGD algorithm. This approach incorporated the Gaussian subsampling mechanism and employed moment accounting techniques, providing a more precise estimate of privacy loss and an accurate method for calculating privacy budgets. However, these privacy guarantees come at the expense of convergence accuracy. The DP mechanism enforces a clipping threshold to constrain the magnitude of updates and adds noise proportional to this threshold to the updated parameters or gradients. Consequently, clipping bias and noise-induced errors are introduced, reducing the model's utility. 

\textbf{DP-preserving optimization strategies}. Most existing research focuses on two key challenges in DP-based learning: gradient clipping bias and noise errors. Regarding clipping bias, Ma et al.~\cite{10210511} proposed an optimized sparse response mechanism for federated learning under DP protection, which filters parameters during upload to control the model’s $\ell_2$ norm. Similarly, the work~\cite{cheng2022differentially} regularized local updates by imposing constraints on the model's norm. Both approaches help reduce clipping bias and, consequently, mitigate the negative impact of the DP mechanism on model performance. To reduce noise errors, Yu et al.~\cite{yu2019differentially} introduced a noise perturbation mechanism where the noise added to model parameters gradually decays during iterations. Yuan et al.~\cite{yuan2023amplitude} proposed a noise amplitude decay method based on a geometric series, while Wei et al.~\cite{wei2023securing} developed an adaptive noise injection mechanism that tracks gradient update trends for each sample. Additionally, the study~\cite{yang2023dynamic} introduced the FedDPM method, which uses the Fisher matrix to select model parameters and eliminate less important ones, thereby reducing noise variance. However, despite efforts to mitigate clipping bias and noise errors, these methods inevitably introduce other limitations. For instance, the parameter filtering techniques may result in incomplete parameter upload and reception, potentially introducing model bias. Similarly, the noise decay strategies may cause higher noise variance in the early iterations, which can lead to convergence issues or even divergence. 

Furthermore, the aforementioned studies may overlook the relationship between the learning rate and noise intensity. Existing research on adaptive step sizes has demonstrated that these techniques can significantly reduce the complexity of step size adjustment \cite{bu2020deep} and improve convergence performance \cite{reddi2020adaptive}. For example, Ge et al.~\cite{ge2019step} employed polynomial decay step sizes, while Wu et al.~\cite{10.1145/3570508} proposed an adaptive gradient method that adjusts the step size based on target accuracy or training time constraints. However, there is limited research investigating learning rate adjustments based on varying noise levels to mitigate the impact of noise. Additionally, a comprehensive theoretical framework to guide dynamic learning rate optimization under different noise conditions remains absent. 

To address these challenges, we propose a novel dynamic learning rate adjustment strategy that effectively manages the relationship between the learning rate and noise variance. This strategy integrates a progressive gradient fusion method that leverages historical gradients to optimize noise variance throughout the training process. Furthermore, we provide theoretical guidance for adjusting the learning rate in response to varying noise intensity, thereby enhancing convergence performance while maintaining robust privacy protection.

\begin{table}[htbp]
	\centering
	\setlength{\tabcolsep}{2pt}
	\renewcommand{\arraystretch}{1.3}
	\caption{Summary of Notations}
	\begin{tabular}{cl}
		 \toprule\arrayrulecolor{black}
		 {Notation} &  {Description} \\
		 \midrule\arrayrulecolor{black}
		 {$\mathcal{G}^t, V, \varepsilon^t$} &  {Time-varying directed graph with node set $V$ and edge set $\varepsilon^t$} \\
		 {$P^t$} &  {Non-negative mixing matrix} \\
		 {$D_i, B_i, b$} &  {Local dataset of node $i$, sampled mini-batch, and single sample} \\
		 {$\epsilon, \delta$} &  {DP parameters} \\
		 {$G', G$} &  {Maximum gradient $\ell_2$-norm and clipping threshold} \\
		 {$\beta_i^t, \alpha_i^t$} &  {Learning rate and noise factor} \\
		 {$n$} &  {Number of nodes} \\
         {$T$} &  {Number of iterations} \\
		 {$\psi, \tau$} &  {Decay coefficient and noise update interval} \\
		 {$\vartheta$} &  {Weighting factor for historical aggregated gradient $\tilde{g}_i^{t}$} \\
		 {$\hat{g}_i^t, \overline{g}_i^t$} &  {Clipped gradient and clipped-noisy gradient} \\
		 {$\tilde{g}_i^t$} &  {Aggregated gradient of $\overline{g}_i^t$} \\
		 {$d_\tau^t$} &  {Bias resulting from gradient staleness} \\
		 {$L, a, m$} &  {Lipschitz constant, variance bound, and related parameters} \\
		 {$x_i^t$} &  {Local model parameter} \\
		 {$w_i^t$} &  {Push-Sum weight} \\
		 {$z_i^t$} &  {Debiasing parameter} \\
		\bottomrule
	\end{tabular}
	\renewcommand{\arraystretch}{1}
\end{table}

\section{Preliminary and Problem Formulation}\label{Section:3}

\textbf{Communication topology}. Consider a time-varying directed graph \( \mathcal{G}^t = (V, \varepsilon^t) \), where \( V \) represents the set of computing nodes and \( \varepsilon^t \) represents the set of directed edges at iteration \( t \). The graph \( \mathcal{G}^t \) is captured by a non-negative mixing matrix \( P^t \in \mathbb{R}^{n \times n} \). If \( (i, j) \in \varepsilon^t \) and \( P_{i,j}^t > 0 \), node \( j \) transmits information to node \( i \) at iteration \( t \); otherwise, no communication occurs between these nodes. We assume that \( P^t \) is column-stochastic, meaning that \( \mathbf{1}_{d}^{\top} P^t = \mathbf{1}_{d} \), where \( \mathbf{1}_{d} \) is a \( d \)-dimensional vector of ones. 

\textbf{Decentralized learning problem}. Let there be $n$ computing nodes, each maintaining a local dataset $D_i$. We denote the loss function for the $b$-th data sample at node~$i$ with respect to the model parameter $x$ as $f_i(x; b)$. For instance, 
$f_i(x; b)$ could represent a cross-entropy loss function. The objective of the decentralized learning algorithm is to minimize the following global objective function across all $n$ nodes:
\begin{equation}
	f(x) :=  \sum_{i=1}^{n} P_i f_i(x), f_i(x)=\frac{1}{\|B_i\|}\sum_{b \in B_i}f_i(x;b),
	\label{eq:my_function}
\end{equation}
where $\|B_i\|$ represents the number of data samples at node $i$, and $P_i$ denotes the weight associated with this node. We evaluate the convergence of the algorithm using the expected $\ell_2$-norm of the gradient, $ E\{\|\nabla f(x)\|\}$, where the expectation is taken over the randomness introduced by the algorithm. 

Before presenting our algorithm, we provide the formal DP definition as outlined by \cite{dwork2006our}.

\noindent\textbf{Definition 1} ($(\epsilon, \delta)$-DP): A randomized algorithm $M$ satisfies $(\epsilon, \delta)$-DP if for any neighboring datasets $D$ and $D'$ that differ by one data point, and for any subset $S$ of the output set of $M$, it holds
\begin{equation*}
	\Pr[M(D) \in S] \leq e^\epsilon \times \Pr[M(D') \in S] + \delta,
\end{equation*}
where $\epsilon$ represents the strength of privacy preservation, and $\delta \in [0,1]$ is the probabilistic error bound.

In DP-preserving machine learning problems, the \(\ell_2\)-norm of the gradient is commonly used to quantify the sensitivity of the output function when Gaussian noise is added. Let \( G' \) denote the maximum value of the gradient's \(\ell_2\)-norm and \( G \) represent the clipping threshold. Without gradient clipping, the variance of the added noise increases by a factor of \((G'/G)^2\) compared to the clipping-based setting. This substantial increase in noise variance can significantly degrade the stability and performance of the model. Thus, gradient clipping is a crucial component of DP mechanisms, as it limits the magnitude of the noise required to ensure privacy. 

However, selecting an appropriate clipping threshold remains a challenging task. An excessively large clipping threshold increases noise intensity and reduces model accuracy, while a carefully reduced threshold can mitigate noise impact and improve training performance. Thus, in this paper, we adopt a gradient clipping approach with a decaying threshold to alleviate the error introduced by the privacy-preserving mechanism and enhance model utility. 

In the context of decentralized computing under time-varying directed topologies, the work~\cite{kempe2003gossip} introduced the SGP method. This approach allows each node to independently select its mixing weight without relying on the other nodes in the network. During the learning process, each node maintains three key variables: the model parameter $x_{i}^{t}$, the scalar Push-Sum weight $w_{i}^{t}$, and the debiasing parameter $z_{i}^{t}$. The update rules for these variables are as follows: $x_{i}^{t+1} = \sum_{j=1}^{n} P_{i,j}^{t} x_{i}^{t+1/2}$, $w_{i}^{t+1} = \sum_{j=1}^{n} P_{i,j}^{t} w_{i}^{t}$, $z_{i}^{t+1} = {x_{i}^{t+1}}/{w_{i}^{t+1}}$. This formulation enables decentralized optimization while addressing communication constraints and ensures convergence even under dynamically changing network topologies. By incorporating the debiasing parameter $z_{i}^{t}$, the method compensates for the bias introduced by non-uniform mixing weights, enhancing both accuracy and stability \cite{pmlr-v97-assran19a}. 

\section{The ADP-VRSGP algorithm}\label{Section:4}
In this section, we introduce two innovative modules in the proposed ADP-VRSGP algorithm: a stepwise noise decaying mechanism combined with a dynamic learning rate (SDLR) and a progressive gradient fusion (PGF) mechanism. 

\subsection{SDLR Strategy}
We propose an SDLR strategy to mitigate the adverse effects of noise on model performance in differentially private mechanisms. Unlike linear and cyclic decay strategies, SDLR aims to achieve a better balance between privacy protection and model accuracy. This approach maintains a stable noise level at fixed intervals and gradually reduces the noise intensity, thereby avoiding the instability associated with continuous decay while still ensuring effective privacy protection. Furthermore, SDLR adjusts the learning rate in response to the varying noise decay rates, which further minimizes the negative impact of noise. The proposed method is elaborated in detail in Algorithm~\ref{algorithm:1}.

Specifically, the noise scale is defined as \( \alpha_i^t \sigma_i \), where \( \alpha_i^t \) is an increasing function representing the noise intensity multiplier and \( \sigma_i \) is the base noise scale. Formally, \( \alpha_i^t \) is defined as \( \alpha_i^{\lfloor t / \tau \rfloor} \), where \( \tau \) represents the time interval for noise variation. During training, we introduce decaying noise in the form of \( \alpha_i^{(T-t)} \sigma_i \), where \( T \) denotes the total number of iterations. This dynamic adjustment of noise intensity ensures that, in the later stages of training, noise perturbation decreases, thereby preserving the accuracy of gradient updates.

Since the noise introduced during the iterative process is influenced by the learning rate, we aim to better manage noise reduction based on the predefined noise intensity coefficient \( \alpha_i^t \). We define the learning rate coefficient as follows:
\begin{equation} \label{learning-rate-coefficient}
	\beta_i^t = 
	\begin{cases} 
		\alpha_i^t \alpha_i^{(T-t)}, & t \leq {\Xi }{T}, \\
		\alpha_i^t \alpha_i^t, & t > {\Xi }{T}. 
	\end{cases} 
\end{equation}
In (\ref{learning-rate-coefficient}), we assume $\Xi \in (0,1)$. As demonstrated in Lemma~\ref{lemma0} to be stated, compared to a single-piece function, the proposed piecewise function effectively reduces noise error, providing a better balance between noise reduction and gradient accuracy, especially in the later stages of the training process.

\begin{lemma}\label{lemma0}
	
	If the learning rate coefficient is set as in (\ref{learning-rate-coefficient}), the bound on the noise error is improved compared to using either \( \beta^t = \alpha^t \alpha^{T-t} \) or \( \beta^t = \alpha^t \alpha^t \). This adjustment results in a more favorable trade-off between noise reduction and model accuracy, particularly in the stages after ${\Xi}{T}$.
\end{lemma}
\begin{proof}
	See Section~\uppercase\expandafter{\romannumeral1} in the supplementary file. 
\end{proof}

Lemma~\ref{lemma0} indicates that, compared to a single-function approach, the piecewise learning rate setting can reduce convergence errors induced by noise, thereby enhancing convergence performance. This effect is particularly pronounced under higher levels of privacy protection.

\begin{algorithm}[htb]
	\caption{SDLR Algorithm}\label{algorithm:1}
	\begin{algorithmic}[1]
		\STATE   \textbf{Input:} The number of iterations $T$, $G> 0$, $\eta$, $\alpha_i$, $\beta_i$, $w_i^0 = 1$, $x_i^0 = z_i^0 = x^0$, privacy budget $(\epsilon _{i}, \delta _{i})$ for all $i \in \{1, 2, \ldots, n\}$.
		\FOR{$t = 1,2,\ldots,T$ at node $i$}
		\STATE   Take a random sample $B_{i}$ with sampling probability $\frac{\|B_i\|}{\|D_i\|}$;
		\STATE   For each $b \subseteq  B_{i}$, compute gradient at $z_i^t$: $\nabla f_{i}(z_{i}^{t})=\frac{1}{\|B_i\|}\sum_{b \in B_i}\nabla f_{i}(z_{i}^{t}, b)$;
		\STATE   $g_i^t=\nabla f_{i}(z_{i}^{t})$;
		\STATE   $\hat{g}_i^t=\min\left\{ 1, \frac{G}{\|g_i^t\|} \right\}g_i^t$;
		\STATE   $\overline{g}_i^t = \hat{g}_i^t + \alpha_i^{T-t}  u_i^t, \quad u_i^t \sim  N
		(0, \sigma_{i}^{2} \mathbf{1}_d)$;
		\STATE   $x_{i}^{t+1/2} \gets x_{i}^{t} - \eta_i^{t}\overline{g}_i^t$, $\eta_i^{t}=\eta/\beta_i^t$; 
		\vspace{1mm}
		\STATE   Send $(x_{i}^{t+1/2}, w_{i}^{t+1/2})$ to out-neighbors;
		\STATE   Receive $(x_{j}^{t+1/2}, w_{j}^{t+1/2})$ from in-neighbors;
		\STATE   $x_{i}^{t+1} \gets \sum_{j=1}^{n} P_{i,j}^{t}x_{j}^{t+1/2}$;
		\STATE   $w_{i}^{t+1} \gets \sum_{j=1}^{n} P_{i,j}^{t}w_{j}^{t+1/2}$;
		\vspace{1mm}
		\STATE   $z_{i}^{t+1} \gets x_{i}^{t+1}/w_{i}^{t+1}$.
		\ENDFOR
	\end{algorithmic}
\end{algorithm}

\subsection{PGF Mechanism}
We further enhance the SDLR strategy by introducing a PGF mechanism for iterative updates. In the early stages of training, excessive noise injected into the process can hinder convergence or even lead to divergence. To mitigate this issue, we aggregate historical gradients, effectively reducing noise variance and improving both stability and efficiency. However, gradients from earlier iterations may become outdated, as older gradients tend to be less relevant than more recent ones. The inclusion of outdated gradients can degrade model performance and potentially cause divergence. The PGF mechanism addresses this challenge by adaptively adjusting the stepwise interval for outdated gradients, thereby minimizing their negative impact on the training process. The details of this mechanism are elaborated in Algorithm~\ref{algorithm:2}.

During the iterative process, the magnitude of the gradients typically decreases. If the clipping threshold is set too large, it can introduce excessive noise, which may negatively impact model performance. To address this, we adopt a decaying clipping threshold strategy, where the threshold is dynamically adjusted throughout the iterations according to \( G = \psi G \), with a decay coefficient \( \psi \). In the \( t \)-th iteration, assume that the time interval includes the previous \( u = (t - 1) \mod \tau \) gradients. We mitigate the noise impact by aggregating these \( u \) gradients with the current gradient from the \( t \)-th iteration. To maximize the effectiveness of this aggregation, the time interval \( \tau \) should be aligned with the stepwise decay interval of the noise coefficient \( \alpha \). This alignment ensures that the noise intensity coefficients remain consistent at each aggregation stage, thereby preventing an increase in noise intensity after aggregation. The aggregated gradient is
\begin{equation}
	\begin{aligned}
        \tilde{g}_i^t = 
	\begin{cases} 
		(1 - \vartheta) \bar{g}_i^{t} + \vartheta \tilde{g}_i^{t-1}, & \alpha^{T-t} = \alpha^{T-t+1}$ and $t \neq 0, \\
		\tilde{g}_i^t, & otherwise, 
	\end{cases}
	\end{aligned}
\end{equation}
where \( \vartheta \) controls the influence of the previous aggregated gradient \( \tilde{g}_i^{t-1} \), and \( \bar{g}_i^{t} \) represents the current gradient at the \( t \)-th iteration.

\begin{algorithm}[htb]
	\caption{SDLR Algorithm with PGF Mechanism}\label{algorithm:2}
	\begin{algorithmic}[1]
		\STATE  \textbf{Input:} The number of iteration $T$, $G> 0$, $\eta$, $\alpha_i$, $\beta_i$, $\vartheta$, $w_i^0 = 1$, $x_i^0 = z_i^0 = x^0$, privacy budget $(\epsilon _{i}, \delta _{i})$ for all $i \in \{1, 2, \ldots, n\}$.
		\FOR{$t = 1,2,\ldots,T$ at node $i$}
		\STATE  Take a random sample $B_{i}$ with sampling probability $\frac{\|B_i\|}{\|D_i\|}$ ;
		\STATE  For each $b \subseteq  B_{i}$, compute gradient at $z_i^t$: $\nabla f_{i}(z_{i}^{t})=\frac{1}{\|B_i\|}\sum_{b \in B_i}\nabla f_{i}(z_{i}^{t}, b)$;
		\STATE  $g_i^t=\nabla f_{i}(z_{i}^{t})$;
		\STATE  $\hat{g}_i^t=\min\left\{ 1, \frac{G}{\|g_i^t\|} \right\}g_i^t$;
		\STATE  $G=\psi G$;
		
		\STATE  $\overline{g}_i^t = \hat{g}_i^t +\alpha_i^{T-t}  u_i^t, \quad u_i^t \sim   N
		(0, \sigma_{i}^{2} \mathbf{1}_d)$;
		\IF{$\alpha^{T-t} = \alpha^{T-t+1}$ and $t \neq 0$}
		\STATE  $\tilde{g}_i^t=(1-\vartheta) 
		\overline{g}_i^t+\vartheta \tilde{g}_i^{t-1}$;		
		\ENDIF

		\STATE  $x_{i}^{t+1/2} = x_{i}^{t}-\eta_i^{t}\tilde{g}_i^t$, $\eta_i^{t}=\eta/\beta_i^t$; 
		\vspace{1mm}
		\STATE  Send $(x_{i}^{t/2}, w_{i}^{t+1/2})$ to out-neighbors;
		\STATE  Receive $(x_{j}^{t/2}, w_{j}^{t+1/2})$ from in-neighbors;
		\STATE  $x_{i}^{t+1}=\sum_{i=1}^{n} P_{i,j}^{t}x_{j}^{t+1/2}$;
		\STATE  $w_{i}^{t+1}=\sum_{i=1}^{n} P_{i,j}^{t}x_{j}^{t+1/2}$;
		\vspace{1mm}
		\STATE  $z_{i}^{t+1}=x_{i}^{t+1}/w_{i}^{t+1}$.
		\ENDFOR
		
	\end{algorithmic}
\end{algorithm}

\subsection{Impact Analysis of the PGF Mechanism}

In this subsection, we analyze the impact of the proposed PGF mechanism. Recall that PGF utilizes \((t-1) \mod \tau\) historical gradients for aggregation to improve convergence speed. However, due to the staleness effect, each updated gradient may contain outdated errors. To address this issue, we aggregate historical gradients using a stepwise decaying interval, selecting up to \( \tau \) gradients at each aggregation step to mitigate the negative impact of stale gradients. In this manner, the gradient can be decomposed into three parts:
\begin{equation}
	\begin{aligned}
		{E}\{\tilde{g}_i^t\} = {E}\{\hat{g}_i^t\} + {E}\{d_\tau^{t}\} + {E}\{\alpha ^{T-t}u_i^{t}\},
	\end{aligned}
\end{equation}
where \({E}\{\tilde{g}_i^t\} \) represents the estimated value of the aggregated gradient. The three components are described as follows:
\begin{itemize}
	
	\item  \( {E}\{\hat{g}_i^t\} \): The estimated value of the clipped gradient in the \( t \)-th iteration, reflecting the computed gradient after clipping and stabilization updates.
	
	\item  \( {E}\{d_\tau^{t}\} \): Capturing the bias introduced by the staleness effect caused by outdated gradients.
	
	\item  \(  {E}\{\alpha ^{T-t}u_i^{t}\} \): Noise Error originating from Gaussian noise introduced by the DP mechanism.
\end{itemize}

\begin{lemma} \label{lemma1}Assume that the error between the historical gradient and the current gradient \( g_i^t \) satisfies 
	
	\[
	\max_{r \in ((t - ((t-1) \mod \tau)), t)} \|\hat{g}_i^t - \hat{g}_i^r\|^2 \leq d_\tau,
	\]
	where the gradient deviation is controlled by \( d_\tau \). We have
	
	\begin{equation}
		\begin{aligned}
			{E}\left\{\|d_\tau^{t}\|^{2}\right\} \leq (\frac{(1-\vartheta)}{(1+\vartheta)} + \frac{2}{(1+\vartheta)}\vartheta^{2\tau-1})d_{\tau}.
		\end{aligned}
	\end{equation}
	
\end{lemma}
\begin{proof}
	See Section~\uppercase\expandafter{\romannumeral2} in the supplementary file.
\end{proof}

As shown in Lemma \ref{lemma1}, the magnitudes of \( \hat{g}_i^t \) and \( \hat{g}_i^r \) can be controlled by appropriately adjusting the clipping threshold, which in turn regulates the magnitude of \( d_\tau \). When \( \hat{g}_i^t \) and \( \hat{g}_i^r \) are subject to the same clipping threshold (or are clipped), \( d_\tau \) will be constrained within that threshold (or approximated to 0). In this way, the previously hard-to-quantify \( d_\tau \) can be transformed into a quantifiable variable through the clipping threshold. This transformation simplifies error control and provides a more intuitive basis for selecting the clipping threshold. By carefully designing and adjusting the clipping threshold, we can optimize the gradient update process without significantly increasing computational complexity, thereby enhancing the overall performance and stability of the algorithm.

\begin{lemma} \label{lemma2}
	The expected squared norm of the error \(\alpha^{T-t}u_i^t\) after applying the PGF mechanism is given by 
	\begin{equation} \nonumber
		\begin{aligned}
			E\!\left[\|\alpha^{T-t}u_i^{t}\|^2\right] \!=\! (\frac{(1-\vartheta)}{(1+\vartheta)} \!+\! \frac{2}{(1+\vartheta)}\vartheta^{2\tau-1} )E\left[\|\alpha_i^{T-t} u_i^{t}\|^2\right].
		\end{aligned}
	\end{equation}
\end{lemma}
\begin{proof}
	See Section~\uppercase\expandafter{\romannumeral3} in the supplementary file.
\end{proof}
From Lemmas~\ref{lemma1} and~\ref{lemma2}, we observe a fundamental trade-off between the staleness bias error and the noise error, which is governed by the hyperparameter \(\vartheta\). In the following, we analyze the impact of \(\vartheta\) on this trade-off:
\begin{itemize}
	\item Small \(\vartheta\). When \(\vartheta\) is small, the current gradient is assigned a greater weight, causing the update process to rely more heavily on the most recent gradient information. While this mitigates the staleness bias error, it also introduces a relatively higher noise error due to the reduced aggregation of historical gradient information.
	
	\item Large \(\vartheta\). When \(\vartheta\) is large, the algorithm tends to aggregate historical gradients, thereby stabilizing updates by averaging out random noise. This effectively reduces noise error but increases staleness bias error, as model updates rely more on outdated information.
	
	\item \(\vartheta = 0\). In this case, the algorithm degenerates into standard gradient descent, where all gradient updates are based solely on the most recent information, eliminating staleness bias error. However, without the aggregation of historical gradients, noise error is significantly amplified.
\end{itemize}

This trade-off suggests that by appropriately tuning \(\vartheta\), one can achieve a balance between mitigating staleness bias error and minimizing noise error, ultimately improving the overall efficiency and convergence performance of the algorithm.

Now, we analyze the impact of \(\tau\). Lemma~\ref{lemma2} indicates that, with a fixed \(\vartheta\), a larger \(\tau\) results in a smaller noise error, as more historical data is used to smooth out the noise. However, this may increase staleness bias error due to a greater reliance on outdated gradients. Conversely, a smaller \(\tau\) reduces staleness bias but increases noise error due to fewer aggregated gradients, making the algorithm more sensitive to noise fluctuations. Lemma~\ref{lemma1} shows that when the clipping threshold is small, the impact of \( d_\tau \) becomes negligible. When the time interval \(\tau\) is large, although the magnitude of the gradients is constrained by the clipping threshold, significant variations may still occur among the individual gradient components. These variations can lead to substantial fluctuations in both the direction and magnitude of the model updates. In contrast, when \(\tau\) is small--meaning over a shorter time interval--the changes in the gradient components are relatively minor, allowing \( d_\tau \) to be approximated as 0. 

This leads to a key question: \textit{How can we determine an appropriate decay interval \( \tau \) that minimizes noise variance while keeping gradient staleness under control?} To address this, we present the following remark. 

\begin{remark}\label{remark}
	From Lemma~\ref{lemma2}, It can be observed that as the time interval \(\tau\) increases, the magnitude \(h=\frac{(1-\vartheta)}{(1+\vartheta)} + \frac{2}{(1+\vartheta)}\vartheta^{2\tau-1}\) decreases, thereby reducing the noise error. However, the magnitude \(h\) is bounded by the constant term \(\frac{(1-\vartheta)}{(1+\vartheta)}\), making it unreasonable to increase \(\tau\) indefinitely under this constraint. Based on this observation, we propose a strategy for setting the time interval \(\tau\). Given the parameter \(\vartheta\), we use \(\frac{(1-\vartheta)}{(1+\vartheta)}\) as a limit to determine \(\tau\) such that 
	\(
	h - \frac{(1-\vartheta)}{(1+\vartheta)} < 0.01.
	\)
	The selected value of \(\tau\) will then serve as the time interval for our analysis.
\end{remark}

\section{Theoretical Analysis}\label{Section:5}
In this section, we analyze the performance guarantee of the proposed ADP-VRSGP algorithm. We begin by making the following assumptions regarding the loss function in problem~\eqref{eq:my_function}.

\begin{assumption}\label{assumption:1}
	For each node \(i\), its local loss function \(f\) has a Lipschitz continuous gradient: 
	\begin{equation*}
		\|\nabla f(x) - \nabla f(y)\| \leq L\|x - y\|, \quad \forall x, y \in \mathbb{R}^d.
	\end{equation*}
\end{assumption}
\begin{assumption}\label{assumption:2} For \( \forall b \in D_i \), there exist two finite positive constants $a$ and $m$ such that 
	\begin{equation*}
		\|\nabla f_i(x) - \nabla f(x)\|^2 \leq a^2,
	\end{equation*}
	\begin{equation*}\frac{1}{n} \sum_{i=1}^{n} \|\nabla f_i(x_{i}^{t},b) - \nabla f(x_{i}^{t})\|^2 \leq m^2.\end{equation*}
\end{assumption}

\begin{assumption} \label{assumption:3}
	Each node is its internal neighbor. There exist two finite positive integers J and $\kappa$, such that for $\forall l \geq 0$, the edge set $ {\textstyle \bigcup_{t=lJ}^{(l+1)J-1 }} \varepsilon ^{t} $ forms a graph that is strongly connected and has a diameter not exceeding $\kappa$.
\end{assumption}

Based on \cite{abadi2016deep}, we observe a negative correlation between noise and privacy loss: the larger the noise variance, the smaller the privacy loss, and vice versa. To achieve the desired level of privacy protection, we present a theorem below that explains how to appropriately adjust the noise parameters.

\begin{theorem} \label{thm1}
	There exist two positive constants \( c_1 \) and \( c_2 \). Let \(\varsigma_i = \frac{\|B_i\|}{\|D_i\|}\). For a given total number of iterations \( T \), if the conditions \( \epsilon_i < c_1\varsigma_i^2 T \), \( \delta_i > 0 \), and 
	
	\begin{equation} 
		\sigma_i = \frac{ Gc_2\varsigma_i\sqrt{\ln(1/\delta_i)}}{\epsilon_i} \sqrt{{\textstyle\sum_{t=0}^{T-1} \frac{1}{(\alpha_i^{t})^2}}}.
	\end{equation}
	are satisfied, then node \( i \) obtains \((\epsilon_i, \delta_i)\)-DP protection, where \( N(0, \sigma_{i}^{2} I_d) \) denotes the added Gaussian noise.
\end{theorem}
\begin{proof}
	See Section~\uppercase\expandafter{\romannumeral4} in the supplementary file.
\end{proof}

If \( \alpha_i^{t} = 1 \), then Theorem~\ref{thm1} simplifies to the same form as Theorem 1 in \cite{abadi2016deep}. We now proceed to analyze the convergence guarantee for the proposed ADP-VRSGP algorithm. 

\begin{theorem}\label{theorem2} Suppose that Assumptions \ref{assumption:1}-\ref{assumption:3} hold. Let \( f^* = \min_{x \in \mathbb{R}} f(x) \) and \( \eta = \frac{K\sqrt{n}}{T^{p}} \), and assume that the condition \( \alpha_i^{t} = \sqrt{K} t^s \) holds for a positive constant \( K \). There exist constants \( C \) and \( q \in [0, 1) \) determined by the network diameter $\kappa$ and the sequence of mixing matrices \( P^t \). If $s={\frac{1}{4} - \frac{p}{2}}$ and the total number of iterations satisfies 
	
	\begin{equation*}
		\begin{aligned}
			T &\geq \max\left\{ 4nL^2, \left(\frac{162nC^{2}L^{2}}{1-q^{2}}\right)^{\frac{2}{1+2p}},(nL^2)^{\frac{2}{1+2p}}, \right. \\
			&\quad\quad\quad\quad\quad \left. n^{\frac{2}{1+2p}},  \left(\frac{2K}{nL}\right)^{\frac{2}{1+2p}}, \left(\frac{9}{\sqrt{n}}\right)^{\frac{4}{3-2p}} \right\},
		\end{aligned}
	\end{equation*}

	 we have
	\begin{equation}
		\begin{aligned}
			&\frac{1}{nT}\Sigma_{t=0}^{T-1}\Sigma_{i=1}^{n}E\{\|\nabla f_{i}(z_{i}^{t})\|^{2}\}\leq \underbrace{ \frac{4F_{0}+5LA_{1}}{\sqrt{nT}}}_{\text{fixed terms}} \\
			&+\underbrace{\frac{4hdL(5A_2+1)M}{\sqrt{nT} }}_{\text{noise error}}+\underbrace{\frac{4 +5LA_{3}}{\sqrt{nT}}(\varrho +\upsilon)}_{\text{stale gradient error and clipping bias}},
		\end{aligned}
		\label{T1}
	\end{equation}
	where \[M=\frac{\eta^2}{n^2} {\textstyle\sum_{i=1}^n \frac{G^{2}c_2^{2}\varsigma_i^{2}\ln \left(1/\delta_i\right)}{\epsilon_i^2}}{\textstyle \sum_{t=0}^{T-1}\frac{(\alpha^{(T-t)})^2}{(\beta ^{t})^2}}{\textstyle \sum_{t=0}^{T-1}\frac{1}{(\alpha ^{t})^2}},\] and \[\upsilon=\frac{1}{n}\textstyle\sum_{i=1}^n{\textstyle \sum_{t=0}^{T-1}} {E}   \left \{ \left\| \left( 1 - \min\left\{ 1, \frac{G}{\|g_i^t\|} \right\} \right) g_i^t \right\|^{2} \right \},\] $A_1=\frac{6 C^2\|x_i^{0}\|+108 C^2 F_0+18 C^2\left(m^2+3 a^2\right)}{(1-q)^2}$, $A_2=\frac{6 C^2\left(n+9L\right)}{(1-q)^2},$ $A_3=\frac{110 C^2}{(1-q)^2}$, $F_0 = f(x_0) - f^*$, $h \!=\!\frac{(1-\vartheta)}{(1+\vartheta)} + \frac{2}{(1+\vartheta)}\vartheta^{2\tau-1}$ and $\varrho ={\textstyle \sum_{t=0}^{T-1}} {E}   \left \{  \|d_{\tau}^t\|^{2} \right \}$.
\end{theorem}
\begin{proof}
	See APPENDIX.
\end{proof}

In highly dynamic or adversarial scenarios—characterized by frequent link failures and malicious packet drops or tampering—the model information aggregated at each node becomes non-uniform, introducing a systematic bias into the optimization process. To counteract this, we introduce a debiasing parameter, $z$, and rigorously characterize its statistical properties. We then derive a tight convergence upper bound for ADP-VRSGP, revealing its pronounced sensitivity to the network structure. Specifically, as the communication topology grows more intricate, the number of iterations required for global information dissemination increases sharply; this escalation is quantified by the connectivity parameter $J$ (Assumption~\ref{assumption:3}). Theorem~\ref{theorem2} captures this effect through the propagation-efficiency parameter $q$ (formally defined in \eqref{15}), which is proportional to $J$. A large $q$ corresponds to sparse, highly dynamic topologies, forcing the iteration complexity $T$ to scale linearly with $q$, whereas a small $q$ reflects an efficient and stable propagation environment. Theorem~\ref{theorem2} thus provides an interpretable rule for adaptively tuning the iteration budget based on the measured $q$, ensuring optimal convergence across diverse network conditions. Moreover, we establish a sharp limiting result: if any subset of agents becomes permanently isolated, $q\to 1$, deriving $T\to\infty$ and causing the algorithm to converge only to a neighborhood solution rather than the global optimum. This finding explicitly delineates the theoretical limitations of distributed learning frameworks that rely on connectivity assumptions.

\begin{remark}\label{remark:2}
	From Theorem~\ref{theorem2}, it is evident that the first term in the numerator of (\ref{T1}) represents the fixed bias introduced by the decentralized learning process, while the subsequent two terms correspond to the noise error, stale gradient error, and clipping bias. Notably, there exists a negative correlation between \( h \) (in the noise error) and \( d_{\tau} \) (in the stale gradient error). Specifically, for a given \( \vartheta \), as the interval $\tau$ for the decaying step size increases, the value of \( h \) decreases, thereby reducing the noise error. However, this reduction in \( h \) simultaneously increases \( d_{\tau} \) in the stale gradient error. Consequently, an optimal balance between \( h \) and \( d_{\tau} \) must be achieved.
\end{remark}

\begin{figure*}[htbp]
\begin{corollary}\label{corollary}
	When \( \tau = 1 \) and under the absence of gradient clipping, the error bound in (\ref{T1}) can be refined through parameter \( p \) as follows: There exist positive constants $H_1$, $H_2$, and $H_3$, such that
	\begin{equation} \label{C1}
		\begin{aligned}
			 \frac{1}{nT}\sum_{t=0}^{T-1}\sum_{i=1}^{n}& E\left[\|\nabla f_{i}(z_{i}^{t})\|^{2}\right] \\
			& \leq \frac{4F_{0}+5LA_{1}}{\sqrt{nT}} + 
			\begin{cases}
				\sqrt{\frac{T}{n}} \cdot \left(\nu H_1^2 \frac{1}{n} {\textstyle\sum_{i=1}^n \frac{G^{2}c_2^{2}\varsigma_i^{2}\ln \left(1/\delta_i\right)}{\epsilon_i^2}} \right) &  p \in (0,1/2), \\
				\frac{(\log T)^2}{\sqrt{nT}} \cdot \left(\nu H_2^2 \frac{1}{n} {\textstyle\sum_{i=1}^n \frac{G^{2}c_2^{2}\varsigma_i^{2}\ln \left(1/\delta_i\right)}{\epsilon_i^2}} \right)  &  p = 0, \\
				\frac{1}{\sqrt{n}T^{\frac12 + 2p}} \cdot \left(\nu  H_3^2 \frac{1}{n} {\textstyle\sum_{i=1}^n \frac{G^{2}c_2^{2}\varsigma_i^{2}\ln \left(1/\delta_i\right)}{\epsilon_i^2}} \right) &  p \in (-1/2,0),
			\end{cases}
		\end{aligned}
	\end{equation}
\end{corollary}
where $\nu =4^2hdL(3A_2+1). $
\begin{proof}
	See Section~\uppercase\expandafter{\romannumeral5} in the supplementary file.
\end{proof}
\end{figure*}

When the gradients are clipped (or when the clipping threshold is smaller than the gradient norm), \( d_{\tau} \) approaches zero. However, a smaller clipping threshold increases the clipping bias, which adversely affects model performance and convergence speed. Interestingly, recent advances in differentially private learning demonstrate that state-of-the-art results are often achieved with very small clipping thresholds. For example, when the clipping threshold is set to 0.1, Li et al.~\cite{li2022large} empirically validated that a clipping threshold of 0.1 yields optimal convergence for GPT-2 and RoBERTa models on NLP benchmarks (e.g., QNLI, MNLI, and SST-2). Similarly, Bu et al.~\cite{NEURIPS2023_8249b30d} and Tramer \& Boneh~\cite{tramer2021differentially} observed the same phenomenon for SimCLR, ResNeXt-29, and SimCLRv2 on CIFAR-10. Further support comes from \cite{kurakin2022toward,mehta2022large}, where ResNet and Vision Transformer models achieved peak convergence on ImageNet under tight clipping. These findings corroborate the robustness of our algorithm’s convergence guarantees in DP-preserving decentralized learning.

In differentially private learning, the primary challenge lies in balancing privacy protection through noise injection with its adverse effects on model convergence and accuracy. The introduction of noise typically amplifies convergence errors, leading to wider convergence bounds and degraded algorithmic performance. Recent work has thus focused on optimizing error bounds in DP mechanisms to mitigate this trade-off. For time-varying directed communication frameworks, the current state-of-the-art noise bound is $O(\sqrt{T/n})$ \cite{ijcai2024p635}. However, as demonstrated by Corollary~\ref{corollary}, our proposed method not only matches this optimal utility bound but also achieves strictly superior bounds: \( O\left(\frac{(\log T)^2}{\sqrt{nT}}\right) \) when \( p = 0 \), and \( O\left(\frac{1}{\sqrt{n}T^{\frac12 + 2p}} \right) \) when \( p \in (-\frac{1}2, 0) \).

\renewcommand{\arraystretch}{1.5}  
\setlength{\tabcolsep}{8pt}

\begin{table*}[htb]
	\caption{Test accuracy (\%) across different datasets under varying privacy budgets.}
	\centering
	\renewcommand{\arraystretch}{1.5} 
	\begin{tabular}{l c c c c c c c c c}
		\hline
		\rule{0pt}{2.5ex}
		\textbf{Dataset} & \textbf{Setting} & \textbf{A(DP)\(^2\)SGD} & \textbf{MAPA} & \textbf{PPSGD} & \textbf{DFedSAM} &  {\textbf{Fed-CDP}} &  {\textbf{FedDPA}} &  {\textbf{FedAN}} & \textbf{ADP-VRSGP} \\
		\hline
		\multirow{4}{*}{MNIST} & $\epsilon = 2$ & 90.92 & 93.08 & 89.05 & 93.06&  {91.61}&  {93.27}&  {92.81}& \textbf{95.29}( {$\uparrow$ 2.02}) \\
		& $\epsilon = 4$ & 92.07 & 94.38 & 90.71 & 94.89&  {92.35}&  {94.85}&  {94.25}& \textbf{96.09}( {$\uparrow$ 1.20}) \\
		& $\epsilon = 6$ & 93.68 & 95.24 & 91.24 &  95.97&  {94.22}&  {95.47}&  {94.93}& \textbf{97.36}( {$\uparrow$ 1.39}) \\
		& $\epsilon = 8$ & 95.37 & 97.40 & 94.25 & 97.15&  {95.53}&  {96.34}&  {96.28}& \textbf{98.04}( {$\uparrow$ 0.89}) \\
		\hline
		\multirow{4}{*}{Fashion-MNIST} & $\epsilon = 2$ & 70.33 & 71.41 & 70.35& 71.37&  {71.61}&  {72.15}&  {71.57}& \textbf{74.61}( {$\uparrow$ 2.46}) \\
		& $\epsilon = 4$ & 73.41 & 73.92 & 71.81 & 74.25&  {73.84}&  {75.09}&  {74.68}& \textbf{77.38}( {$\uparrow$ 2.29}) \\
		& $\epsilon = 6$ & 75.95 & 76.36 & 74.55 & 77.86&  {76.37}&  {78.54}&  {76.96}& \textbf{80.52}( {$\uparrow$ 1.98}) \\
		& $\epsilon = 8$ & 79.89 & 80.93 & 80.34 & 81.69&  {80.49}&  {82.62}&  {80.35}& \textbf{84.15}( {$\uparrow$ 1.53}) \\
		\hline
		\multirow{4}{*}{CIFAR-10} & $\epsilon = 2$ & 35.47 & 43.67 & 44.81& 45.36&  {44.26}&  {45.66}&  {45.18}& \textbf{47.27}( {$\uparrow$ 1.61})\\
		& $\epsilon = 4$ & 38.35 & 46.98 & 47.53 & 48.22&  {47.13}&  {48.31}&  {47.29}& \textbf{49.35}( {$\uparrow$ 1.04}) \\
		& $\epsilon = 6$ & 40.47 & 50.67 & 51.82 & 50.63&  {52.08}&  {52.54}&  {51.83}& \textbf{53.56}( {$\uparrow$ 1.02}) \\
		& $\epsilon = 8$ & 47.36 & 56.05 & 55.29 & 56.93&  {55.31}&  {57.72}&  {56.79}& \textbf{58.56}({$\uparrow$ 0.84}) \\
		\hline
	\end{tabular}
	\label{tab:1}
\end{table*}

\section{Experimental Evaluation}\label{Section:6}
\subsection{Experiments Setup} \label{ex_set}
\textbf{Datasets and models}. We use Convolutional Neural Networks (CNN) as the base model to evaluate the performance of different algorithms across multiple datasets: MNIST, Fashion-MNIST, and CIFAR-10. Each dataset is split into a training set and a test set, with the training set containing 60,000 images(50,000 for CIFAR-10) and the test set containing 10,000 images. The MNIST dataset consists of handwritten digit images, each 28x28 pixels in grayscale, primarily used for digit classification tasks. The Fashion-MNIST dataset consists of clothing images across 10 categories, with each image being a 28x28 pixel grayscale image, designed to challenge image classification tasks. The CIFAR-10 dataset includes 10 categories of color images, with each category containing 6,000 images of size 32x32 pixels, and is widely used in image recognition and computer vision research. For the MNIST and Fashion-MNIST datasets, we design a CNN model with two convolutional layers and two fully connected layers. For the CIFAR-10 dataset, we design a CNN model with three convolutional layers and two fully connected layers.

\textbf{Baselines}. We evaluate our proposed ADP-VRSGP algorithm against four state-of-the-art approaches: A(DP)\(^2\)SGD~\cite{9524471}, MAPA~\cite{10316599}, PPSGD~\cite{bietti2022personalization}, DFedSAM~\cite{shi2023improving}, Fed-CDP~\cite{wei2023securing}, FedDPA~\cite{yang2023dynamic} and FedAN~\cite{xue2023differentially}. A(DP)\(^2\)SGD implements a Gaussian mechanism in decentralized communications to guarantee sample-level DP. MAPA employs a multi-stage adaptive DP algorithm to enhance model utility. PPSGD improves the diversity between local and global models to mitigate noise-induced errors. DFedSAM leverages sharpness-aware minimization to develop robust local flat models that maintain stability under Gaussian noise perturbations. Fed-CDP uses dynamic clipping bounds and gradient norms for dynamic noise injection. FedDPA leverages Fisher information-based sparsification to reduce noise variance. FedAN dynamically adjusts noise levels based on local and global historical information.

\textbf{Implementation details}. We deploy \( n = 8 \) computational nodes in a decentralized network topology, where each node maintains connections with two immediate neighbors. Our experimental configuration uses: clipping threshold \( G=0.1 \) with decay factor \( \psi =0.99\), local batch size of 64, and privacy specification \( \delta = 10^{-5} \). The data heterogeneity is modeled via Dirichlet distribution with concentration parameter \( \alpha \). We set the total number of iterations to \( T = 1000 \) and define the noise intensity coefficient as \( \alpha_i = \left( \lfloor {t}/{\tau} \rfloor + 10 \right)^s \), where \( \tau = 5 \), \( s = 0.2 \), and \( \vartheta = 0.5 \). All implementations are developed in PyTorch and executed on NVIDIA A100 GPU hardware.

\subsection{Performance Comparison with Baselines}
We conduct a systematic comparative analysis of the test accuracy between our proposed ADP-VRSGP algorithm and baseline methods across varying privacy budgets (\(\epsilon = 2, 4, 6, 8\)). Table~\ref{tab:1} summarizes the results on MNIST, Fashion-MNIST, and CIFAR-10, revealing consistent performance advantages. Specifically, for increasing \(\epsilon\) values, ADP-VRSGP achieves accuracy improvements over the baselines of 2.02\%, 1.20\%, 1.39\%, and 0.89\% on MNIST; 2.46\%, 2.29\%, 1.98\%, and 1.53\% on Fashion-MNIST; and 1.61\%, 1.04\%, 1.02\%, and 0.84\% on CIFAR-10. Notably, ADP-VRSGP maintains robust performance even under stringent privacy constraints (\(\epsilon = 2\)), where noise interference is more pronounced. This demonstrates the algorithm's dual capability to: 1) preserve strong privacy guarantees, while 2) maintaining model accuracy by minimizing noise-induced optimization errors. The consistent performance gains across all test datasets and privacy levels validate ADP-VRSGP's effectiveness in balancing the privacy-utility trade-off---a critical requirement for practical privacy-preserving decentralized learning scenarios.

\begin{figure*}[htb]
	\centering
	\subfloat[$s=0.2$]{\includegraphics[width=2in]{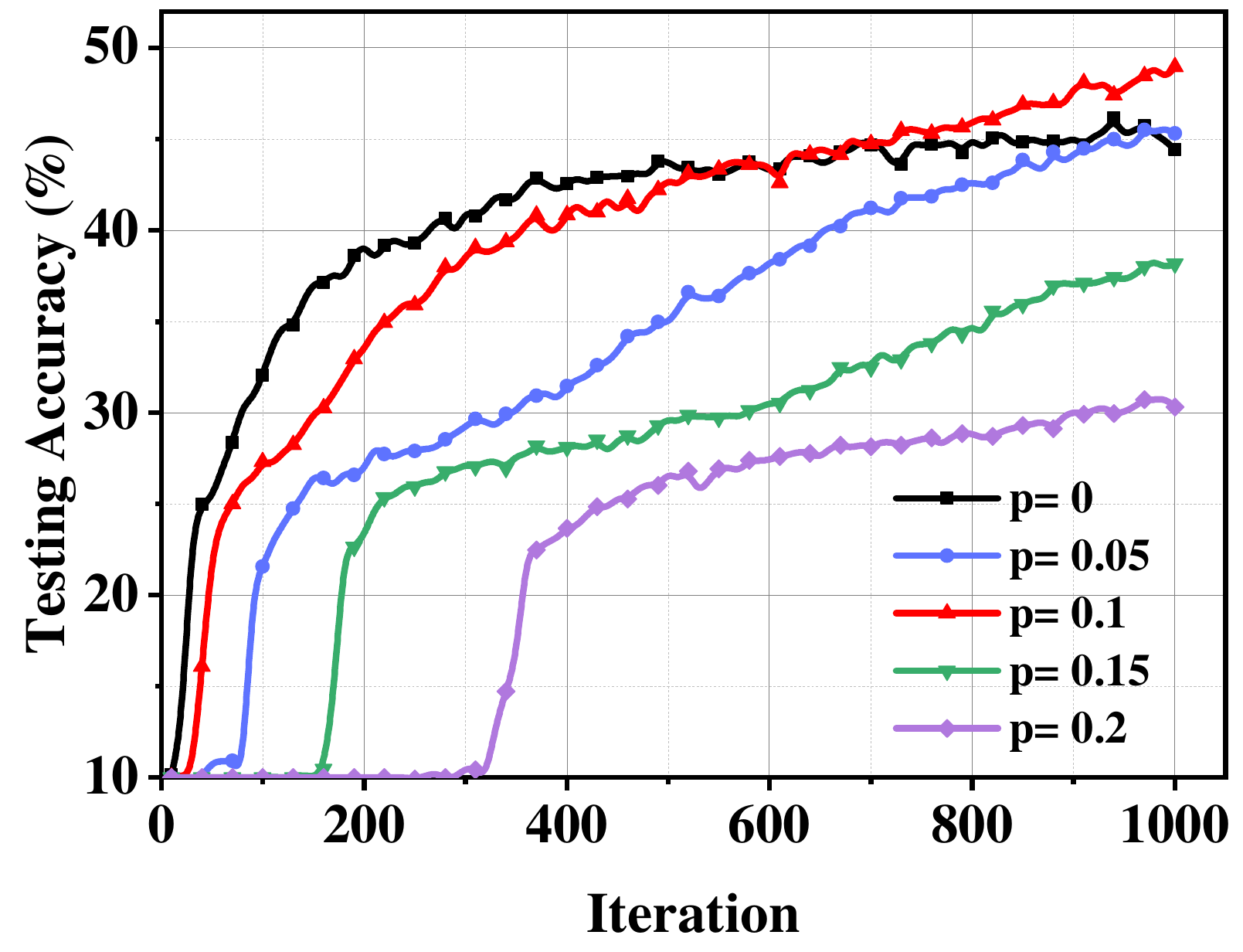}} \hfil
	\subfloat[$s=0.25$]{\includegraphics[width=2in]{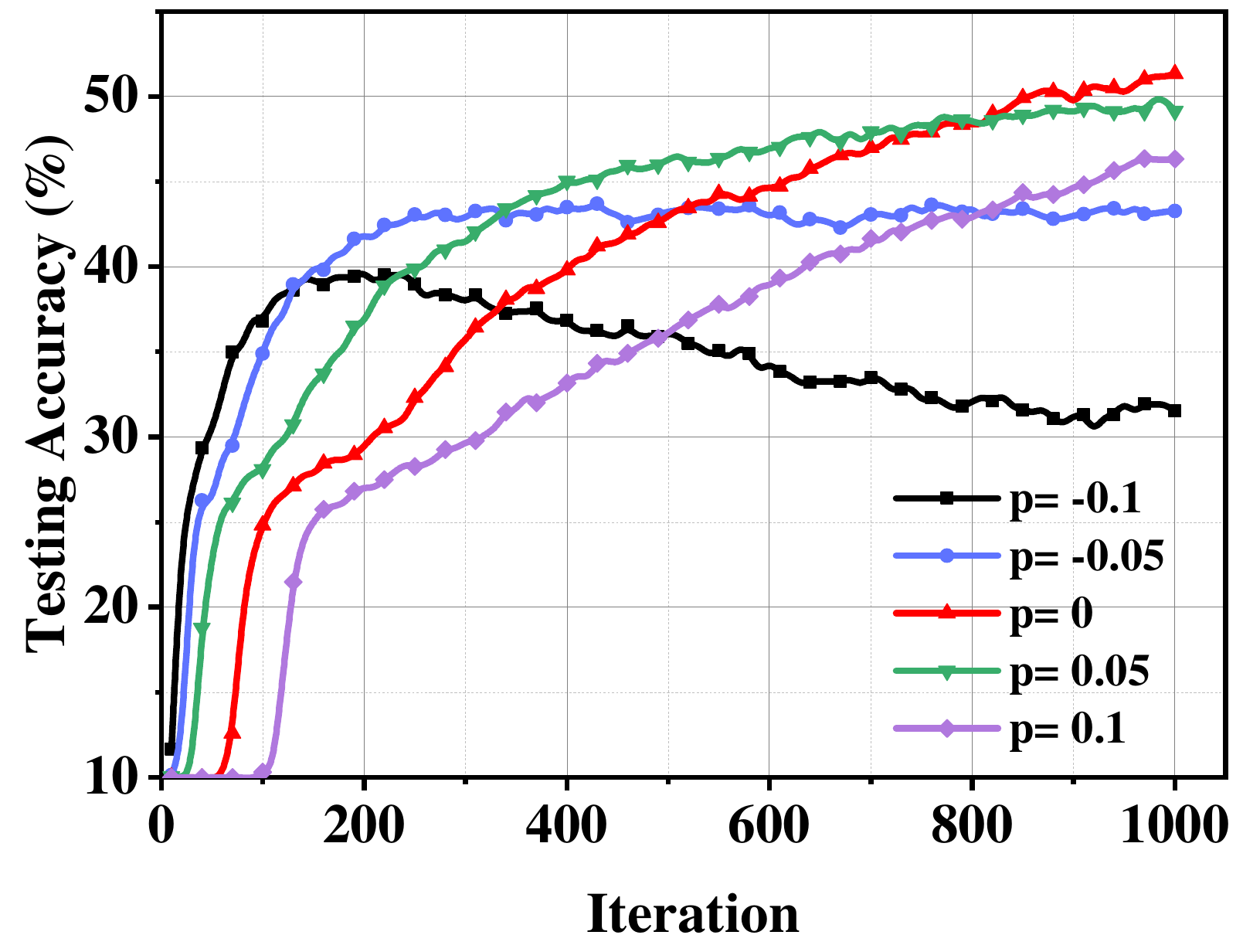}} \hfil
	\subfloat[$s=0.3$]{\includegraphics[width=2in]{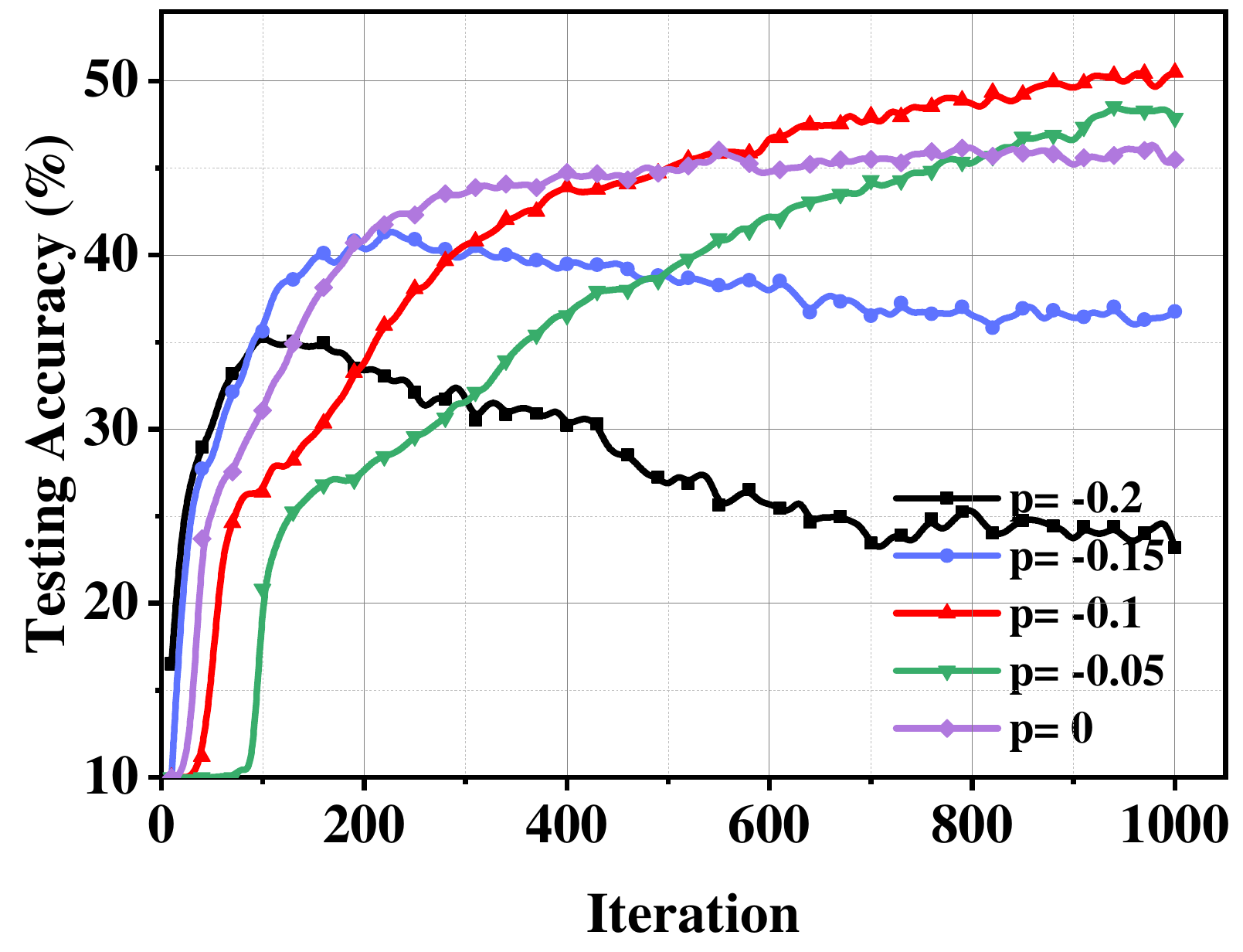}} \\
	\subfloat[$p/(s=0.2)$]{\includegraphics[width=2in]{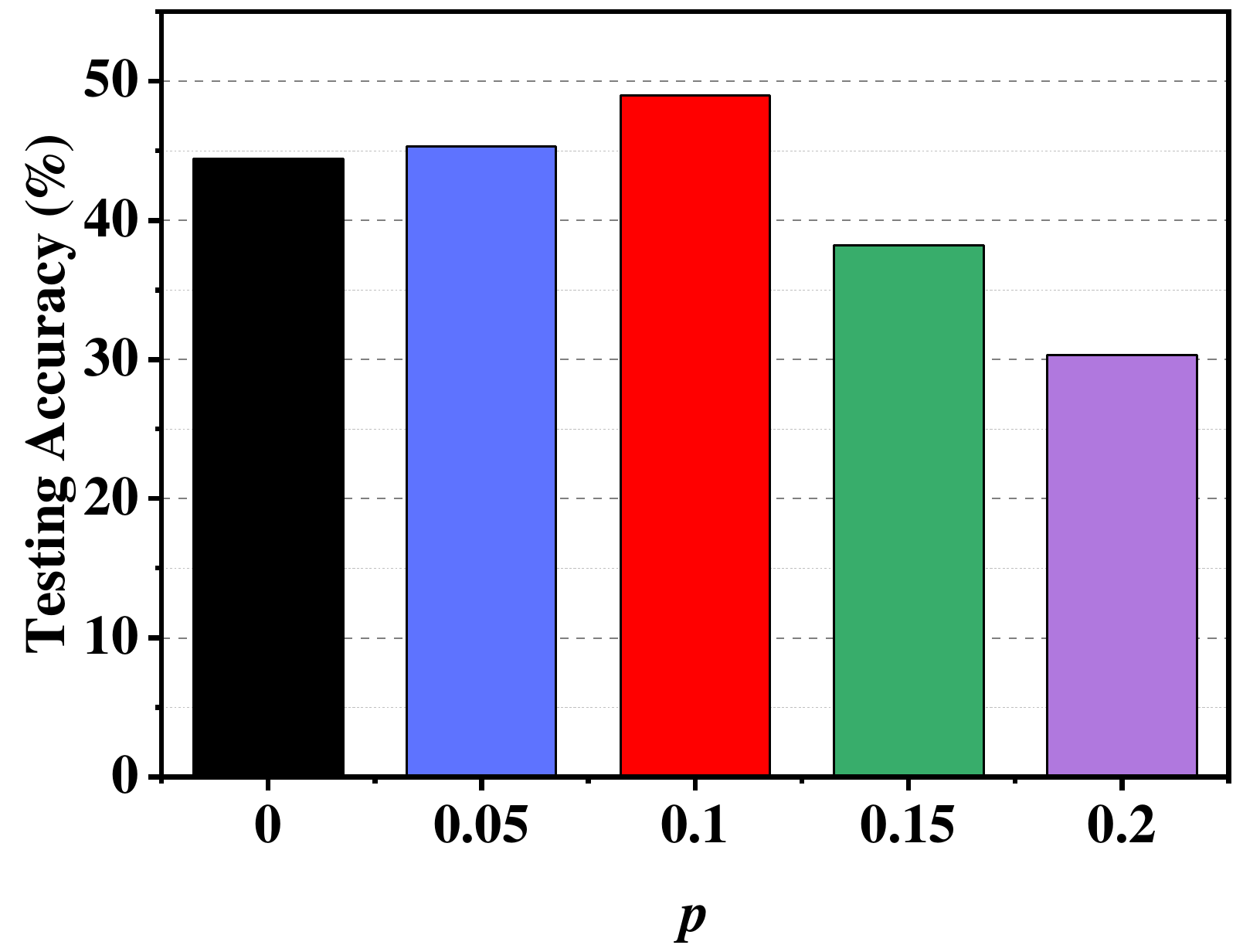}} \hfil
	\subfloat[$p/(s=0.25)$]{\includegraphics[width=2in]{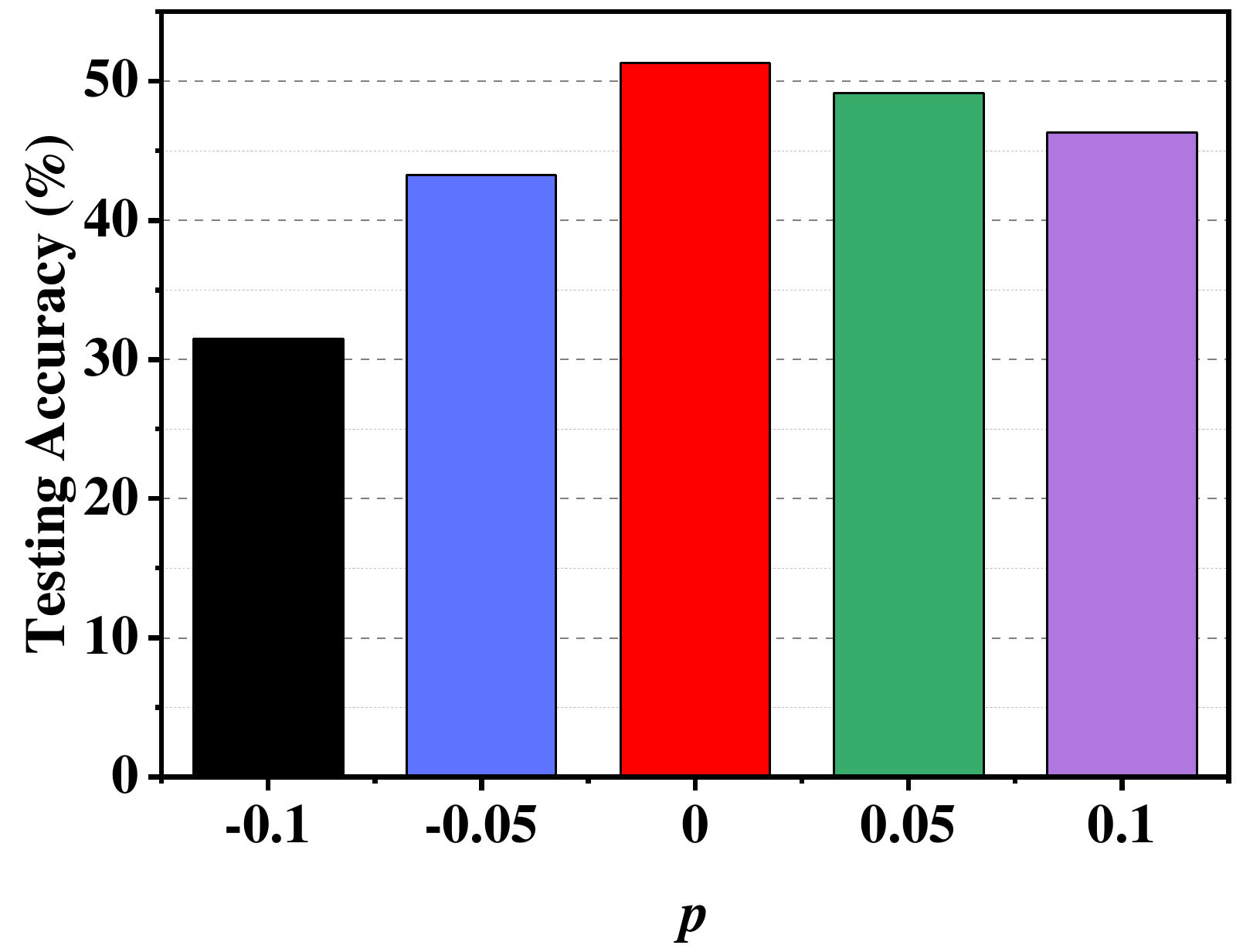}} \hfil
	\subfloat[$p/(s=0.3)$]{\includegraphics[width=2in]{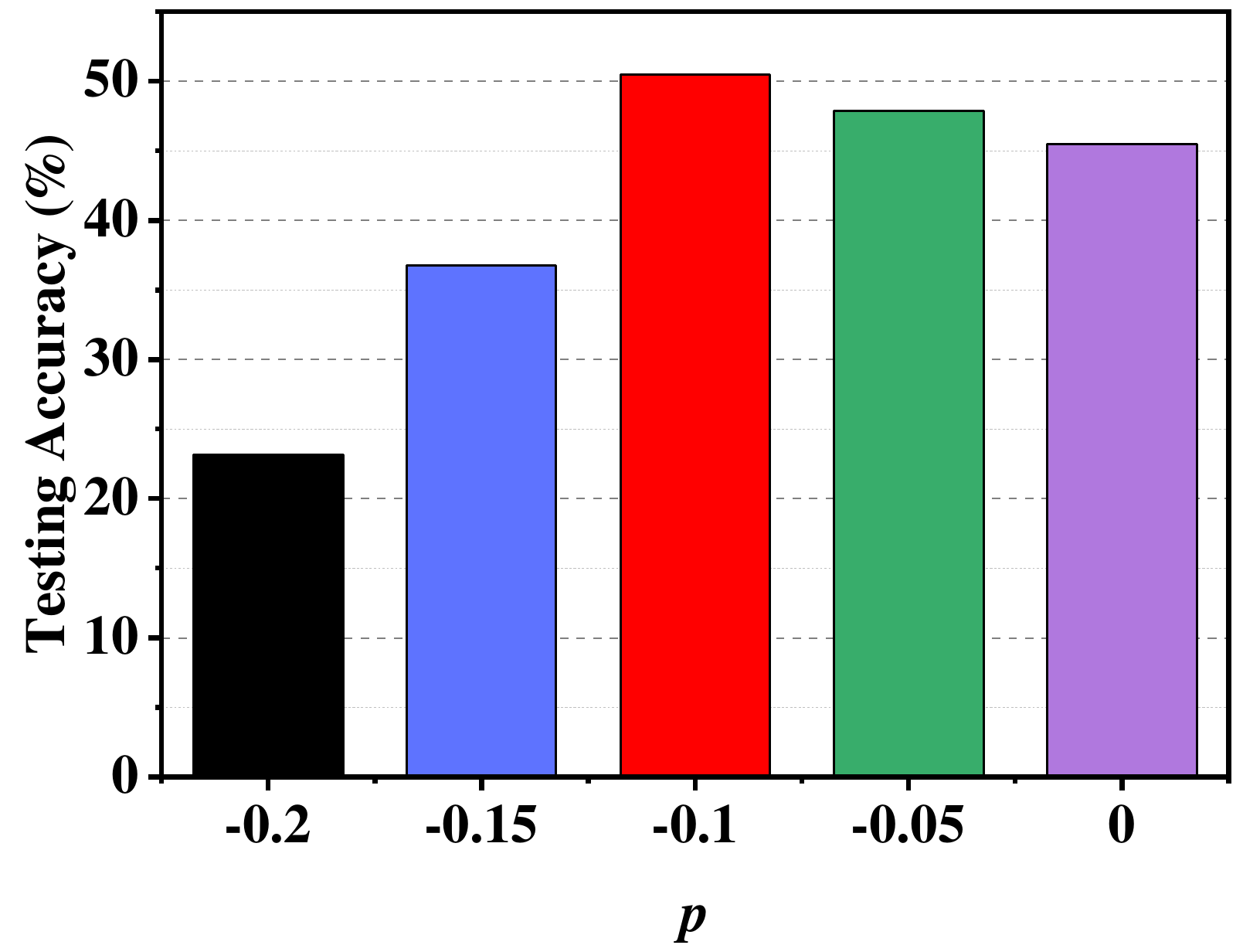}}
	\caption{Test accuracy (\%) on the CIFAR-10 dataset under different learning rates.}
	\label{fig:figure1}
\end{figure*}
\begin{figure*}[htb]
	\centering
	\subfloat[MNIST]{\includegraphics[width=2in]{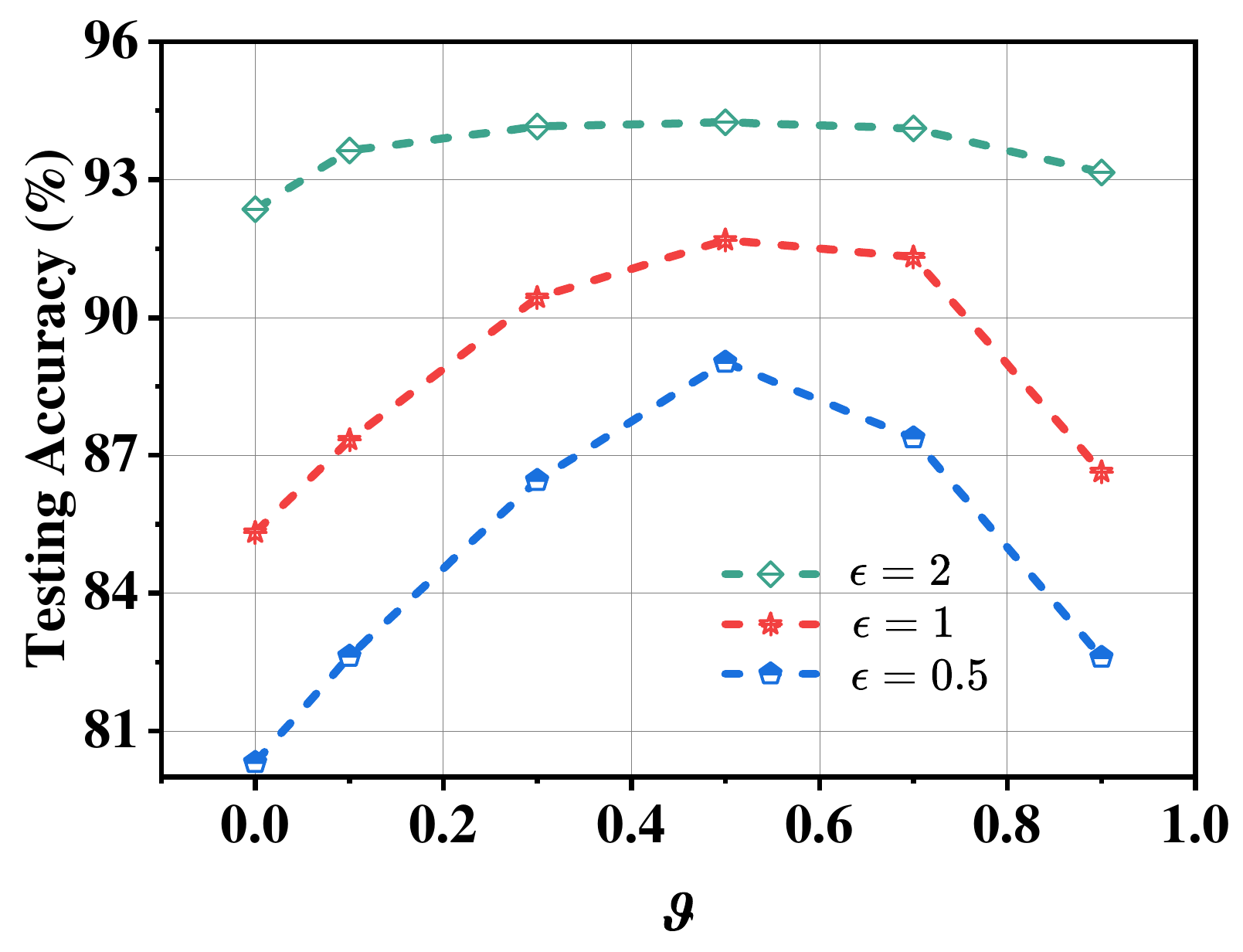}} \hfil
	\subfloat[Fashion-MNIST]{\includegraphics[width=2in]{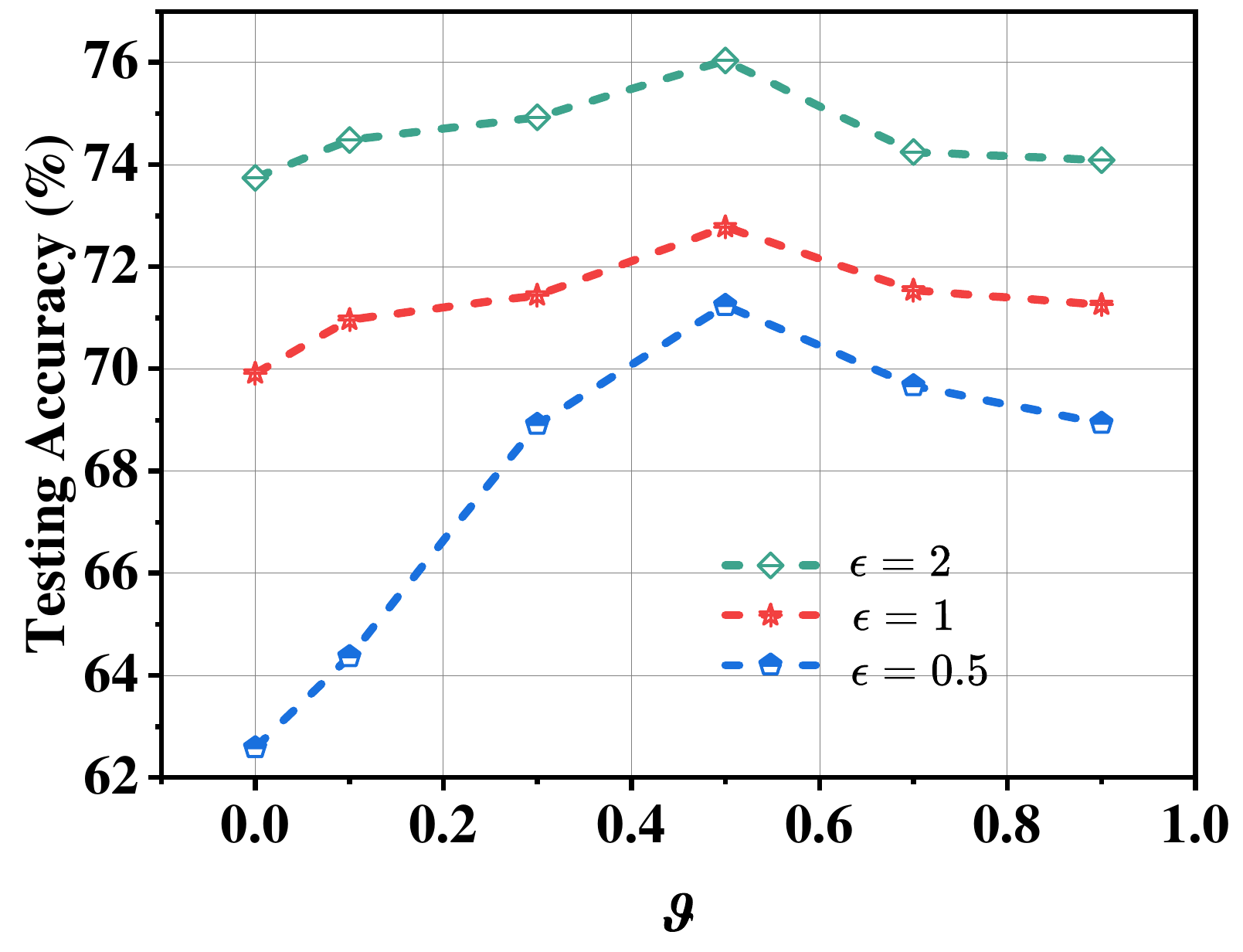}} \hfil
	\subfloat[CIFAR-10]{\includegraphics[width=2in]{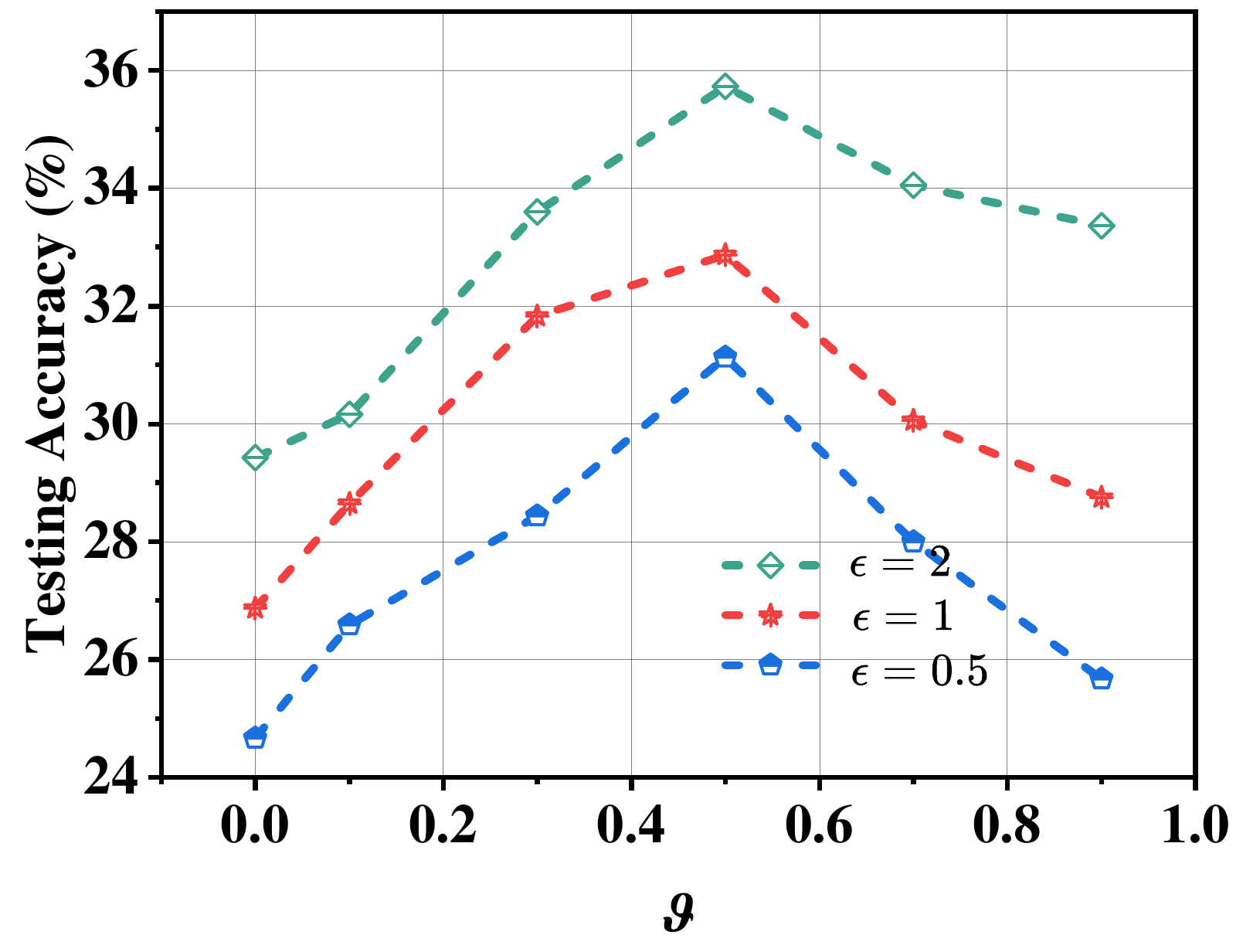}}\hfil
	\subfloat[MNIST]{\includegraphics[width=2in]{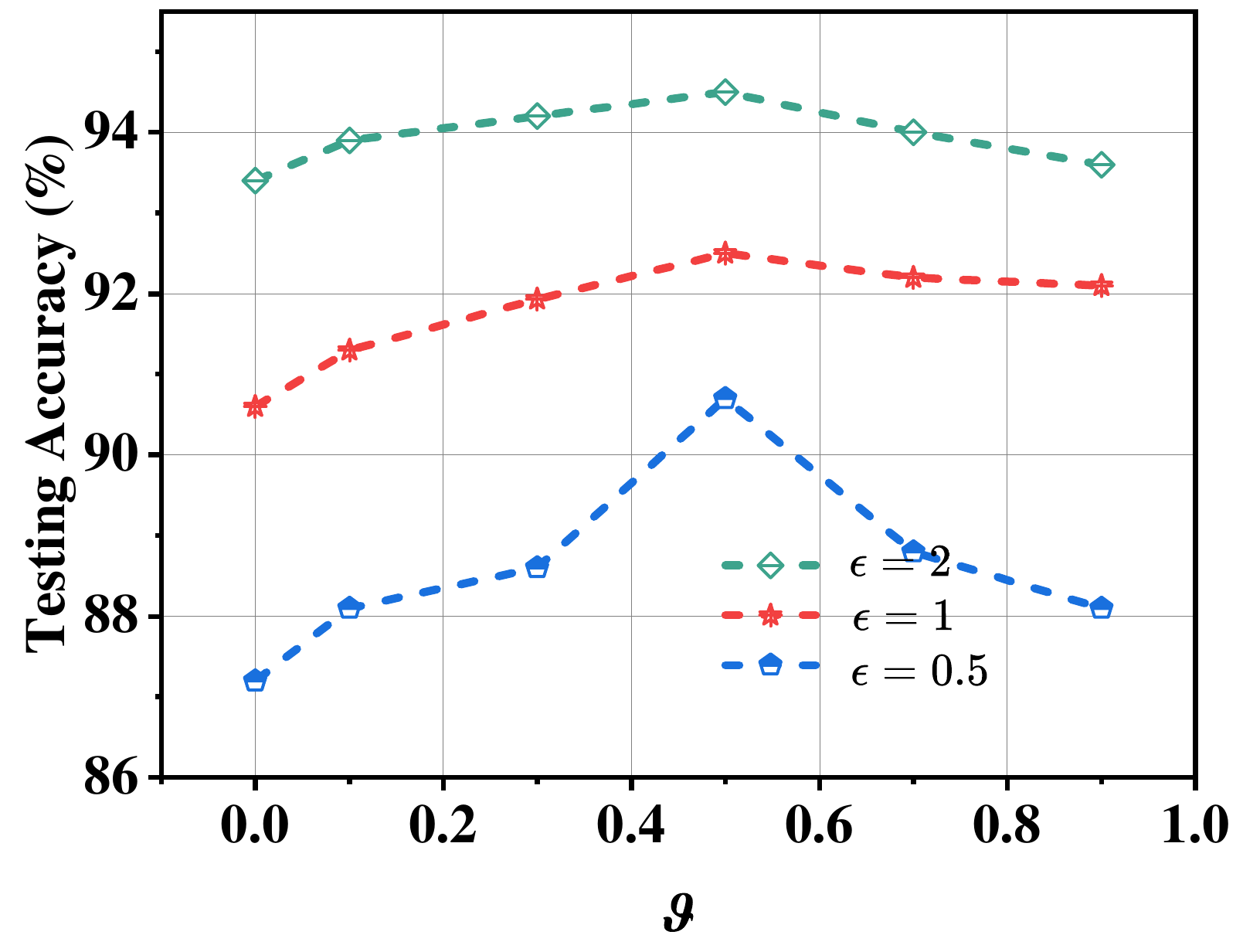}}\hfil
	\subfloat[Fashion-MNIST]{\includegraphics[width=2in]{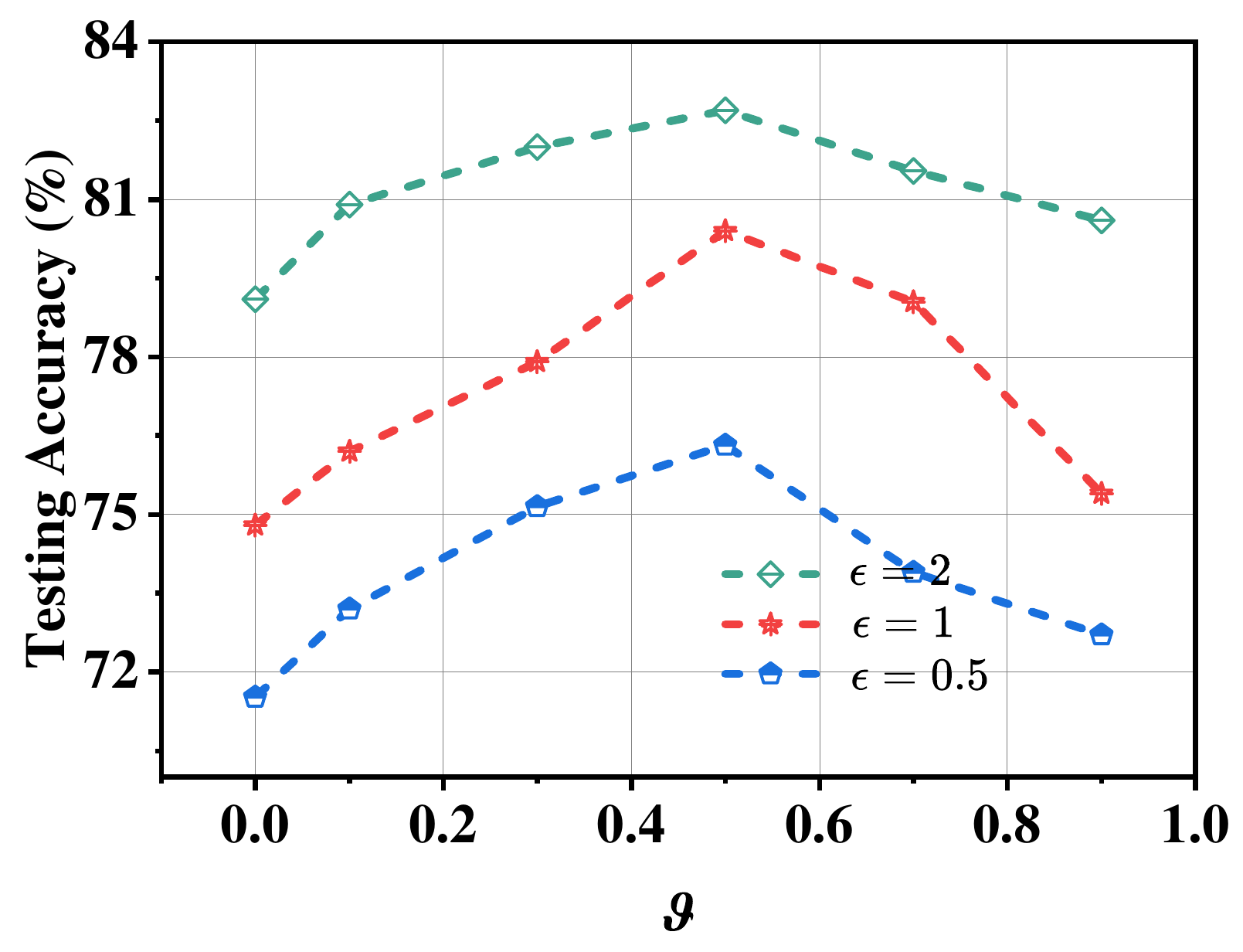}}\hfil
	\subfloat[CIFAR-10]{\includegraphics[width=2in]{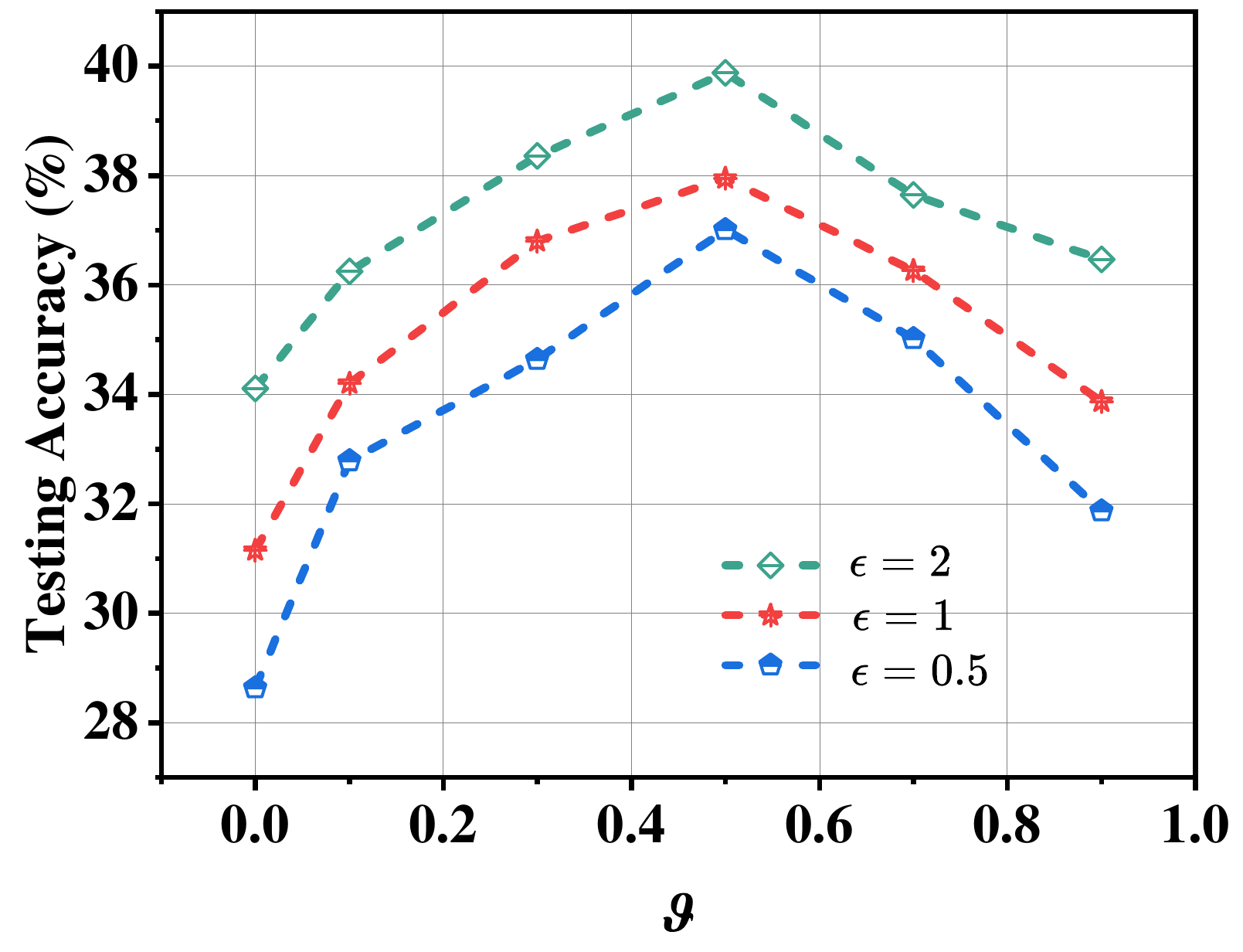}}
	\caption{Test accuracy (\%) versus hyperparameter \(\vartheta\) for different datasets. Upper panels: CNN results; lower panels: ResNet-18 results.}
	\label{fig:figure2}
\end{figure*}

\renewcommand{\arraystretch}{1.5}  
\setlength{\tabcolsep}{10pt}
\begin{table*}[htb]
	\caption{Test accuracy (\%) across different datasets under varying time interval values.}
	\centering
	\begin{tabular}{c c cccc cccc}  
		\hline
		Model & $\vartheta$ & Dataset & $\tau=1$ & $\tau=4$ & $\tau=6$ & $\tau=9$ & $\tau=12$ & $\tau=15$ & $\tau=30$ \\
		\hline
		\multirow{6}{*}{CNN} 
		& \multirow{2}{*}{$\vartheta=0.3$} 
		&MNIST& 79.21 & 85.07 & 86.24 & 88.34 & \textbf{89.38} & 86.71 & 85.75 \\
		& &Fashion-MNIST& 68.23 & 70.75 & 71.09 & 71.58 & \textbf{72.28} & 71.84 & 70.22 \\
		& \multirow{2}{*}{$\vartheta=0.5$} 
		&MNIST& 79.71 & 85.02 & \textbf{89.51} & 88.13 & 87.19 & 87.09 & 86.58 \\
		& &Fashion-MNIST& 68.58 & 71.29 & \textbf{72.92} & 72.01 & 71.55 & 70.97 & 70.05 \\
		& \multirow{2}{*}{$\vartheta=0.7$} 
		&MNIST& 80.28 & \textbf{87.61} & 87.41  & 87.41 & 87.29 & 86.85 & 86.76 \\
		& &Fashion-MNIST& 68.49 & \textbf{72.38} & 71.41  & 71.26 & 71.19 &71.15 & 70.83 \\
		\hline
		\multirow{6}{*}{ {ResNet-18}} 
		& \multirow{2}{*}{ {$\vartheta=0.3$}} 
		& {MNIST}&  {82.34} &  {85.75} &  {88.64} &  {89.27} &  {\textbf{90.59}} &  {89.67} &  {89.62} \\
		& & {Fashion-MNIST}&  {72.49} &  {75.83} &  {77.91} &  {78.39} &  {\textbf{78.61}} &  {77.39} &  {76.38} \\
		& \multirow{2}{*}{ {$\vartheta=0.5$}} 
		& {MNIST}&  {83.53} &  {89.29} &  {\textbf{90.73}} &  {89.72} &  {88.39} &  {87.06} &  {86.61} \\
		& & {Fashion-MNIST}&  {75.46} &  {79.58} &  {\textbf{80.49}} &  {79.52} &  {78.49} &  {78.03} &  {77.86} \\
		& \multirow{2}{*}{ {$\vartheta=0.7$}} 
		& {MNIST}&  {84.23} &  {\textbf{89.48}} &  {88.83}  &  {87.33} &  {87.14} & {86.42} &  {86.23} \\
		& & {Fashion-MNIST}&  {77.28} &  {\textbf{80.03}} &  {79.41}  &  {78.51} &  {77.65} & {77.36} &  {76.72} \\
		\hline
	\end{tabular}
	\label{tab:2}
\end{table*}

\subsection{Optimal Learning Rate under Varying Noise Decay Rates}

We examine the critical relationship between learning rate selection and noise decay rates \( s \) in our ADP-VRSGP algorithm. Through systematic experiments on CIFAR-10 with \( s \in\{0.2, 0.25, 0.3 \}\), we evaluate the algorithm's sensitivity to learning rate adjustments while maintaining node-specific privacy budgets (\( \epsilon = 6 \)) and noise intensity coefficient \( \alpha_i = ({\lfloor t / 5 \rfloor} + 10)^{s} \)). Our theoretical analysis in Theorem~\ref{theorem2} and Corollary~\ref{corollary} establishes the optimal learning rate configuration as: \( p = -\frac{1}{2} - 2s \), which yields distinct optimal values \( p^* = 0.1, 0, -0.1 \) for the three cases outlined in Corollary~\ref{corollary}. In addition to the optimal choices, we also examine other learning rates for comparative analysis: \( p^* - 0.1 \), \( p^* - 0.05 \), \( p^* + 0.05 \), and \( p^* + 0.1 \). 

Fig.~\ref{fig:figure1} shows our theoretically optimal learning rates \( p^* = 0.1, 0, -0.1 \) achieve superior convergence and the best test accuracy, while deviations consistently degrade performance: smaller rates (\( p^* +\{0.1, 0.05\}\)) slow convergence and hinder performance, whereas larger rates (\( p^* -\{0.1, 0.05\}\)) amplify noise and destabilize training. Notably, DP-VRSGP achieves better convergence at $p^*=0$ and $p^*=-0.1$ than at \( p^* = 0.1 \), which represents the current state-of-the-art noise bound. These results validate our theoretical analysis, showing that \( p^*\) optimally balances convergence speed with noise sensitivity under privacy constraints while improving noise bounds, achieving the best privacy-accuracy trade-off. Experiments show learning rate choice critically impacts model performance in differentially private decentralized learning.

\subsection{Effect of Hyperparamter \( \vartheta \) and Time Interval \(\tau\)}
To validate the robustness of our method across different environments, we evaluate the sensitivity of parameters $\vartheta$ and $\tau$ on both CNN and ResNet-10 models. For CNN, we adhere to the original settings, while for ResNet-10, we simulate a more heterogeneous environment. In this setting, nodes are indexed from $0$ to $n-1$. At iteration $t$, node $i$ selects its neighbors according to a periodic rule defined by the edge set: 
$$
\varepsilon^t = \{(i + j \cdot 2^{(t \bmod \log_2 n)}) \bmod n \mid j = 0, \ldots, \frac{n}{2}-1\}.
$$

\vspace{10pt}
\textbf{Effect of \( \vartheta \)}.
We systematically evaluate the impact of the hyperparameter \( \vartheta \in \{0, 0.1, 0.3, 0.5, 0.7, 0.9\}\) on convergence performance across multiple datasets with \( T = 300 \) iterations. Fig.~\ref{fig:figure2} reveals three key findings: (1) Non-zero \( \vartheta \) values consistently improve test accuracy compared to \( \vartheta = 0 \); (2) excessive \( \vartheta \) introduces significant stale gradient bias, while insufficient \( \vartheta \) amplifies noise error, both degrading convergence; and (3) the optimal balance occurs at intermediate values (e.g., \( \vartheta = 0.5 \)), which yields particularly strong gains under strict privacy constraints. For example, on MNIST, accuracy improves by 2.5\%--8.7\%, 1.9\%--6.3\%, and 1.1\% --1.9\% for \( \epsilon = 0.5, 1, 2 \), respectively. Comparable gains are observed on Fashion-MNIST (5.6\%--8.6\%, 3.6\%--2.8\%, and 2.4\%--2.8\%) and CIFAR-10 (6.5\%--8.5\%, 6.0\%--6.8\%, and 5.7\%--6.3\%). These results indicate an enhanced privacy-utility trade-off, with advantages becoming more pronounced under stronger privacy guarantees (i.e., smaller \( \epsilon\) values). Moreover, the convergence trends are highly consistent across both CNN and ResNet-18 models, validating the effectiveness of ADP-VRSGP under DP constraints and highlighting its robustness to hyperparameter variations in heterogeneous networks.

\textbf{Effect of \( \tau \)}.
We evaluate the impact of the time interval $\tau \in \{1, 4, 6, 9, 12, 15, 30\}$ on algorithm performance across different model architectures, using MNIST (\( \epsilon = 0.5 \)) and Fashion-MNIST (\( \epsilon = 1 \)) with three $\vartheta$ configurations (0.3, 0.5, and 0.7). As shown in Table~\ref{tab:2}, 
the theoretically optimal intervals from Remark~\ref{remark}—\( \tau = 12 \) for \( \vartheta = 0.3 \), \( \tau = 6 \) for \( \vartheta = 0.5 \), and \( \tau = 4 \) for \( \vartheta = 0.7 \)—consistently achieve peak test accuracy for both CNN and ResNet-18 models. These optimal values outperform all other \( \tau \) settings across experimental conditions, highlighting the algorithm's robustness in highly heterogeneous networks. This optimal balance stems from competing effects: larger \( \tau \) increases gradient staleness while smaller \( \tau \) amplifies noise error. These results demonstrate that the theoretically derived \( \tau \) optimally trades off these opposing factors, validating our convergence analysis in Theorem~\ref{theorem2}.

\begin{figure*}[htb]
	\centering
	\subfloat[Fashion-MNIST]{\includegraphics[width=2.7in]{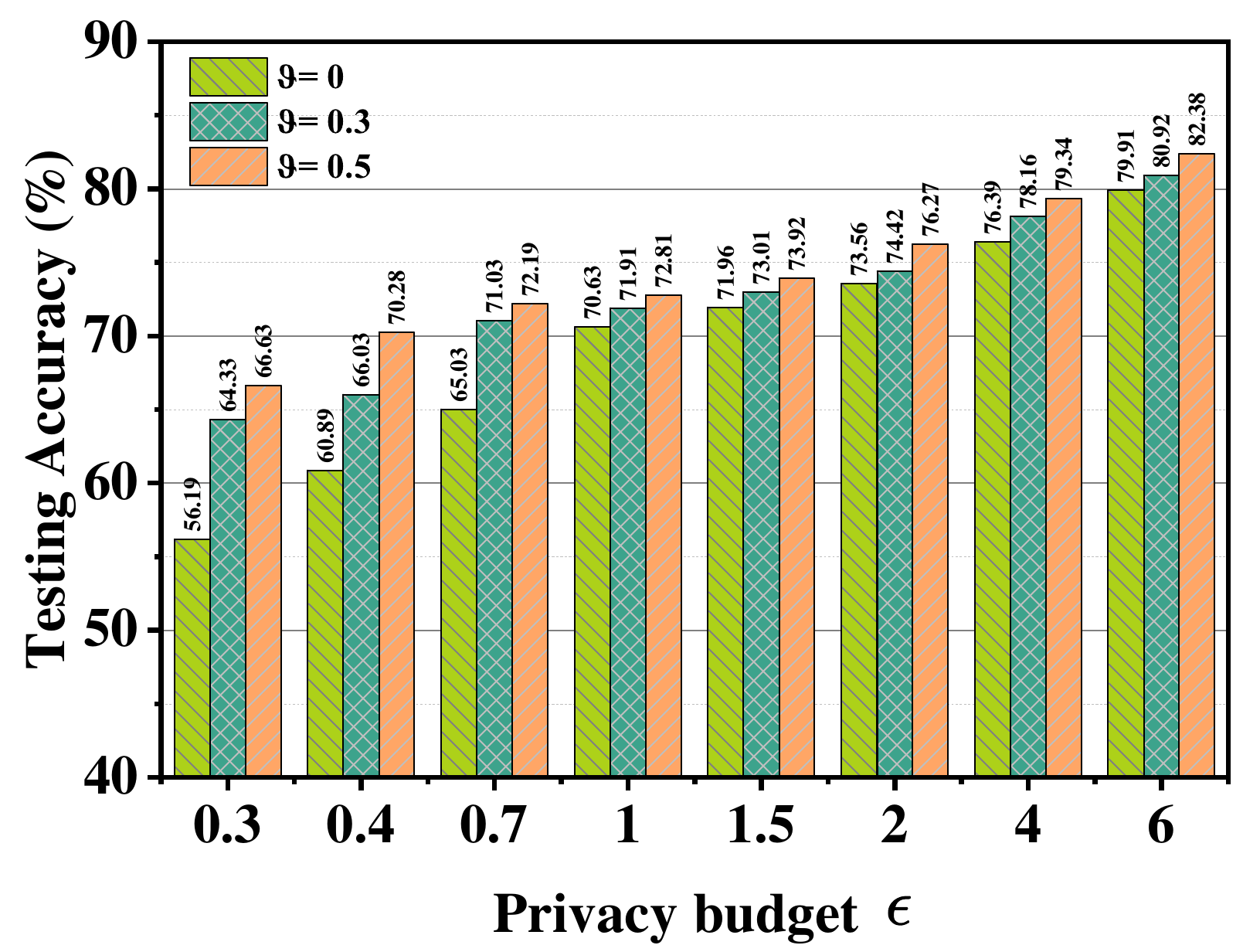}} \hfil
	\subfloat[CIFAR-10]{\includegraphics[width=2.7in]{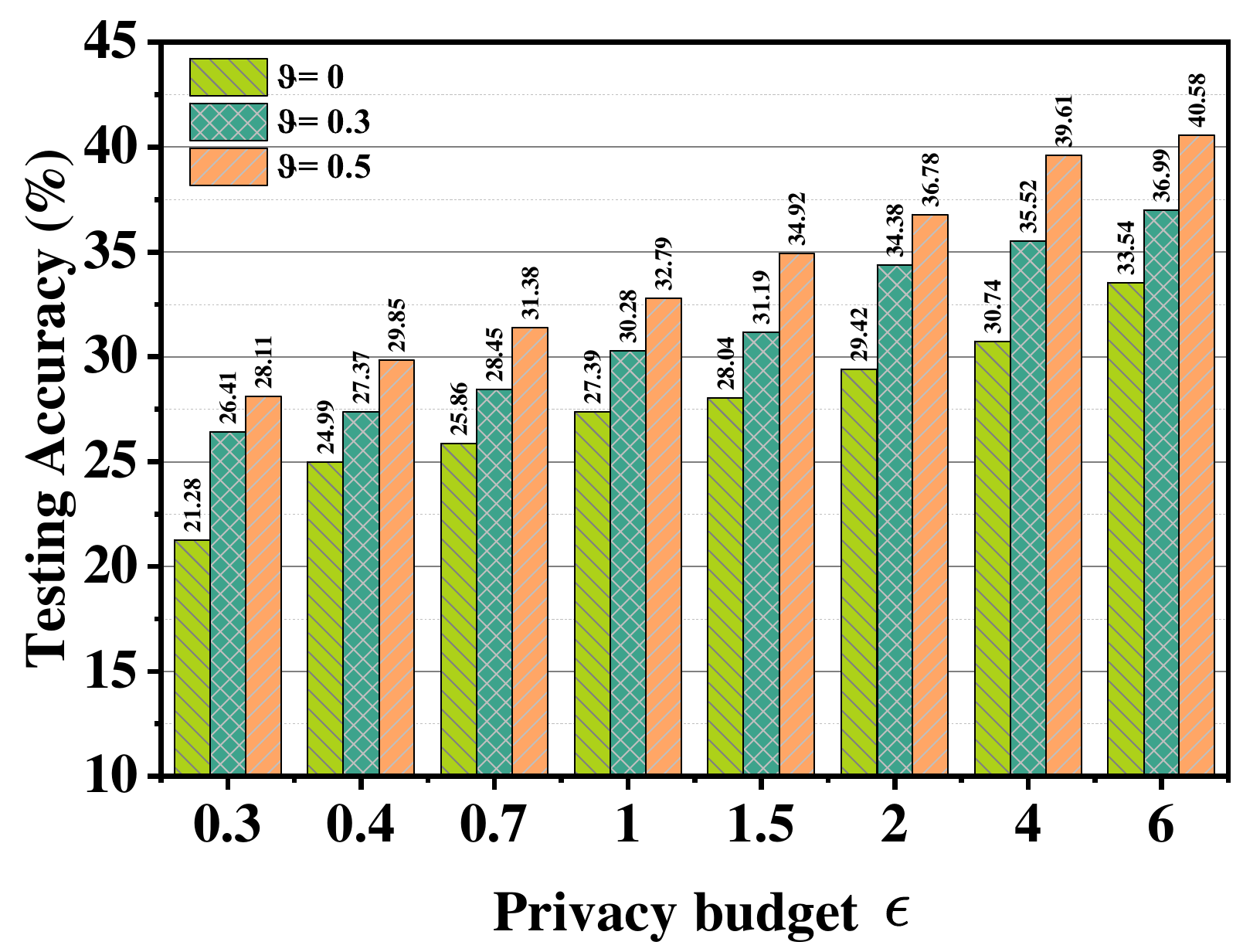}}
	\caption{Test accuracy (\%) on Fashion-MNIST and CIFAR-10 under varying privacy budgets.}
	
	\label{fig:figure3}
\end{figure*}

\renewcommand{\arraystretch}{1.5}  
\setlength{\tabcolsep}{6pt}
\begin{table*}[htb]
	\caption{Performance Comparison (\%) Across Different Numbers of Nodes.}
	\centering
	\begin{tabular}{c c c c c c c c c}
		\hline
		 {\textbf{Nodes}} &  {\textbf{A(DP)\(^2\)SGD}} & 
		 {\textbf{MAPA}} &  {\textbf{PPSGD}} & 
		 {\textbf{DFedSAM}} &  {\textbf{Fed-CDP}} & 
		 {\textbf{FedDPA}} &  {\textbf{FedAN}} & 
		 {\textbf{ADP-VRSGP}} \\
		\midrule
		 {$n=8$} &  {40.47} &  {50.67} & 
		 {51.82} &  {50.63} &  {52.08} &  {52.54} &  {51.83} &  {\textbf{53.56}} \\
		\midrule
		 {$n=16$} &  {39.08} &  {48.22} & 
		 {48.68} &  {46.75} &  {49.77} &  {48.25} &  {49.04} &  {\textbf{52.05}} \\
		\midrule
		 {$n=32$} &  {36.15} &  {46.16} & 
		 {47.76} &  {44.51} &  {48.73} &  {43.12} &  {47.52} &  {\textbf{50.36}} \\
		\midrule
		 {$n=64$} &  {32.71} &  {43.11} & 
		 {44.92} &  {42.45} &  {45.78} &  {40.87} &  {44.98}&  {\textbf{47.89}} \\
		\midrule
		 {$n=128$} &  {31.48} &  {41.86} & 
		 {40.74} &  {39.84} &  {43.32} &  {37.85} &  {42.30}&  {\textbf{45.71}} \\
		\hline
	\end{tabular}
	\normalsize
	\renewcommand{\arraystretch}{1}
	\label{tab:4}
\end{table*}

\renewcommand{\arraystretch}{1.5} 
\setlength{\tabcolsep}{10pt}
\begin{table*}[htb]
	\caption{Test accuracy (\%) on CIFAR-10 under different degrees of data heterogeneity.}
	\centering
	\begin{tabular}{l ccc ccc ccc}  
		\toprule
		Methods  & $\alpha=0.1$ & $\alpha=0.5$ & $\alpha=1$ & $\alpha=1.5$ & $\alpha=2$ \\
		\midrule
		A(DP)\(^2\)SGD & 28.70 & 35.48 & 38.44 & 41.91 & 42.15 \\
		
		MAPA & 35.14 & 43.61 & 47.27 & 48.85 & 49.02 \\
		
		PPSGD & 34.62 & 42.13 & 47.94 & 49.58 & 49.82 \\
		
		DFedSAM & 35.25 & 43.87 & 48.49 & 50.25 & 50.85 \\
		
		\hline 
		ADP-VRSGP & \textbf{37.48}($\uparrow$ 2.59) & \textbf{45.12}($\uparrow$ 1.25) & \textbf{50.15}($\uparrow$ 1.66) & \textbf{53.01}($\uparrow$ 2.76) & \textbf{53.42}($\uparrow$ 2.57) \\
		\hline
	\end{tabular}
	
	\label{tab:3}
\end{table*}

We also conduct an evaluation of PGF's performance across privacy budgets \( \epsilon \in \{0.3, 0.4, 0.7, 1, 1.5, 2, 4, 6\}\), comparing configurations with \( \vartheta = 0.3 \) and \( \vartheta = 0.5 \) (using their respective optimal \( \tau \) values from Remark~\ref{remark}) against the baseline \( \vartheta = 0 \). Fig.~\ref{fig:figure3} demonstrates that: 1) PGF consistently outperforms the \( \vartheta = 0 \) case across different privacy levels; and 2) the performance advantage becomes more pronounced with increasing privacy levels. These results validate PGF's enhanced capability to maintain model utility while providing rigorous privacy guarantees. 

\subsection{Effect of Node Scalability}
We evaluate the impact of network size on CIFAR-10 ($\epsilon = 6$), with results presented in Table~\ref{tab:4}. In small-scale settings (8 and 16 nodes), the model exhibits superior convergence and accuracy, benefiting from sufficient local data at each node for reliable gradient estimation. As the number of nodes increases, the local data volume per node decreases, reducing gradient reliability. This data sparsity, amplified by DP noise, leads to slower convergence and lower final accuracy for a fixed number of iterations. Consequently, large-scale environments require more training iterations to maintain model utility under the joint constraints of data sparsity and privacy.

\renewcommand{\arraystretch}{1.5} 
\setlength{\tabcolsep}{10pt}
\begin{table}[htb]
	\caption{Hyperparameter Selection ($\psi$ and $G$)}
	\centering	
	\begin{tabular}{c c c c c}
		\hline
		 {$\psi$} &  {0.9} & {0.95} & 
		 {0.99} &  {1} \\
		\hline
		 {Accuracy} &  {37.24} & 
		 {43.51} &  {47.27} & 
		 {46.83} \\
		\hline
	\end{tabular}

	\vspace{10pt}
	
	\begin{tabular}{c c c c c c}
		\hline
		 {$G$} &  {0.01} & 
		 {0.05} &  {0.1} & 
		 {0.5} &  {1} \\
		\hline
		 {Accuracy} &  {40.06} & 
		 {45.56} &  {47.27} & 
		 {44.65} &  {36.62} \\
		\hline
	\end{tabular}
	
	\renewcommand{\arraystretch}{1}
	\label{tab:11}
\end{table}

\subsection{Impact of Data Heterogeneity}
In this subsection, we evaluate our algorithm's robustness to varying data heterogeneity levels on CIFAR-10, controlling the degree of heterogeneous data partitioning through the Dirichlet concentration parameter \( \alpha \) (with fixed privacy budget \( \epsilon = 6 \)). Note that data heterogeneity can cause the local models of individual nodes to evolve in divergent directions during training, leading to significant discrepancies in feature representations, training dynamics, and gradient updates. Table~\ref{tab:3} presents two key findings that validate Theorem~\ref{theorem2}: 1) Performance degrades by 13-16\% as \( \alpha \) decreases from 2 to 0.1, reflecting increasing heterogeneity; 2) our algorithm maintains higher test accuracy than baselines across all \( \alpha \) values. The introduction of Gaussian noise further exacerbates these discrepancies between the local and global models, thereby increasing model bias. In scenarios with high data heterogeneity, the combination of Gaussian noise and inconsistencies in local data exacerbates the deviation in gradient directions during the model update process, amplifying the accumulation of model bias. Our variance reduction strategy mitigates this effect, effectively improving model accuracy. 

\subsection{Hyperparameter Sensitivity Analysis}
We evaluate the impact of key parameters on experimental results. The privacy parameter $\delta$  is fixed at $10^{-5}$ following mainstream literature. Sensitivity analyses for $\epsilon, \tau, s, \vartheta$ are presented in Tables~\ref{tab:1}--\ref{tab:2} and Figs.~\ref{fig:figure1}--\ref{fig:figure3}. For the gradient clipping threshold $G$ and the decay factor $\psi$, Table~\ref{tab:11} reports results for $(G \in \{0.01, 0.05, 0.1, 0.5, 1\}$ and $\psi \in \{0.9, 0.95, 0.99, 1\})$, identifying $G=0.1$ and $\psi=0.99$ as optimal. Excessively large $G$ introduces higher noise error, while overly small $G$ causes severe clipping bias, revealing a clear utility trade-off. All experiments adopt this optimal setting to ensure fairness.

\section{Conclusion}\label{Section:7}
This paper proposes ADP-VRSGP, an approach designed for privacy-preserving decentralized learning under time-varying directed communication topologies. By optimizing noise intensity and employing segmented learning rates, the proposed method effectively mitigates noise-induced errors while improving both model performance and convergence speed. Furthermore, ADP-VRSGP leverages historical gradients to address challenges such as slow convergence and potential divergence caused by excessive noise injection during early training stages. We provide a convergence analysis of ADP-VRSGP, deriving an improved noise bound under appropriate learning rate conditions. Experimental results demonstrate the approach's efficacy and robustness, showcasing superior performance across diverse test scenarios.  

Future work will focus on two key directions: 1) exploring alternative noise-decay strategies to optimize training through dynamic noise adjustment; and 2) investigating the integration of historical gradient aggregation with momentum methods or exponential moving averages to enable adaptive learning rate regulation, thereby further enhancing the efficiency and stability of model updates.

{\appendix[Proof of Theorem~\ref{theorem2}] 
         By Algorithm~2, we have
         \begin{equation}				
		\begin{aligned}
         x_{i}^{t+1/2} &= x_{i}^{t}-\eta_i^{t}\tilde{g}_i^t, \eta_i^{t}=\eta/\beta_i^t,\\
         x_{i}^{t+1}&=\sum_{i=1}^{n} P_{i,j}^{t}x_{j}^{t+1/2}.
        \end{aligned}
	\end{equation}
	
        Since the matrix \( P \) is column-stochastic, we further obtain
        \begin{equation}				
		\begin{aligned}
         \frac{1}{n}\sum_{i=1}^{n}x_i^{t+1}
         =\frac{1}{n}\sum_{i=1}^{n}x_i^{t}-\eta_i^{t}\left(\frac{1}{n}\sum_{i=1}^{n}\tilde{g}_i^{t}\right).
        \end{aligned}
	\end{equation}

	If Assumption~1 holds, we have 
	\begin{equation}				
		\begin{aligned}
			f\left(\frac{1}{n}\sum_{i=1}^{n}x_i^{t+1}\right)&\leq f\left(\frac{1}{n}\sum_{i=1}^{n}x_i^{t}\right)\\
			&-\frac{\eta}{\beta_i^{t}}\left\langle\nabla f\left(\frac{1}{n}\sum_{i=1}^{n}x_i^{t+1}\right),\frac{1}{n}\sum_{i=1}^{n}\tilde{g}_i^{t}\right\rangle\\
			&+\frac{L\eta^{2}}{2\left(\beta_i^{t}\right)^{2}}\left\|\frac{1}{n}\sum_{i=1}^{n}\tilde{g}_i^{t}\right\|^{2}.
		\end{aligned}
		\label{7}
	\end{equation}
	
	Taking the expectation with respect to the randomness of $b$, we derive
	\begin{equation}
		\begin{aligned}
			{E}_{b}  &\left \{ f\left(\frac{1 }{n}\textstyle\sum_{i=1}^nx_i^{t+1}\right) \right \}  \le  f\left(\frac{1}{n}\textstyle\sum_{i=1}^nx_i^{t}\right)\\
			&+\frac{L \eta^2}{2\left(\beta_i^t\right)^2}{E}_{b}\left \{ \left\|\frac{1}{n}\textstyle\sum_{i=1}^n\tilde g_i^t \right\|^2 \right \}  \\
			&-\frac{\eta}{\beta_i^t}{E}_{b}\left\{ \left\langle \nabla f\left(\frac{1 }{n}\textstyle\sum_{i=1}^nx_i^{t+1}\right), \frac{1 }{n}\textstyle\sum_{i=1}^n\tilde g_i^t\right\rangle \right\}.
		\end{aligned}
		\label{8}
	\end{equation}
	
	For the error caused by gradient clipping, we define
	\begin{equation*}
		E_{b}\{\Delta_i^{t}\} = {E}\left\| \left( 1 - \min\left\{ 1, \frac{G}{\|g_i^t\|} \right\} \right) g_i^t \right\|^{2}=
		\upsilon_i^t.
	\end{equation*}
	
	The term $\tilde g_i^t$ is further decomposed into
	\begin{equation}
		\begin{aligned}
			&E_{b} \left\{\tilde g_i^t\right\}  =E_b \left\{(g_{i}^{t}+d_\tau^{t}+\upsilon_i^t)+\alpha_i^{T-t} u_i^t\right\}.
		\end{aligned}
		\label{9}
	\end{equation}
	
	Denote the expectation with respect to the randomness of $N(0,\sigma_i^2I)$ as ${E}_{u}$. Given the conditions of $b$ and $N(0,\sigma_i^2I)$. we have 
	\begin{equation}
		\begin{aligned}
			E_{J}& \left\{\left\|\frac{1}{n}\textstyle\sum_{i=1}^n \tilde{g}_i^t\right\|^2 \right\}
			\\
			&=E_{J}\left\{\left\|\frac{1}{n}\textstyle\sum_{i=1}^n
			(g_{i}^{t}+d_\tau^{t}+\upsilon_i^t)+\frac{1}{n}\textstyle\sum_{i=1}^n \alpha_i^{T-t} u_i^t\right\|^2\right\} \\
			&=E_{J}\left\{\left\|\frac{1}{n}\textstyle\sum_{i=1}^n (g_{i}^{t}+d_\tau^{t}+\upsilon_i^t)\right\|^2\right\} \\
			& \quad \quad+E_{J}\left\{\left\|\frac{1}{n}\textstyle\sum_{i=1}^n \alpha_i^{T-t}u_i^t\right\|^2\right\},	
		\end{aligned}
		\label{10}
	\end{equation}
	where ${E}_{J}={E}_{b}{E}_{u}$.

	Next, we obtain 
	\begin{equation}
		\begin{aligned}		
			E_{b}\left\langle\frac{1}{n}\textstyle\sum_{i=1}^n \alpha_i^{T-t} u_i^t,  \nabla f\left(\frac{1 }{n}\textstyle\sum_{i=1}^nx_i^{t+1}\right)\right\rangle=0
		\end{aligned}.
		\label{11}
	\end{equation}
	
	By defining $\mathrm {O} ^{t}= \frac{1}{n} {\textstyle \sum_{i=1}^{n}{E}    \left \{  \left    \| z_{i}^{t}-\bar{x} ^{t}    \right \|^{2}  \right \}}$,  where $\bar{x} ^{t}=\frac{1 }{n}\textstyle\sum_{i=1}^nx_i^{t}$. Under the condition $\frac{\eta}{\beta^t} \leq \min \left\{1, \frac{1}{L} \right\}$, we simplify (\ref{8}) by substituting (\ref{9}), (\ref{10}), and (\ref{11}) into (\ref{8}), yielding
	\begin{equation}
		\begin{aligned}
			\begin{aligned}&E_{J}\left\{f\left(\frac{1}{n}\sum_{i=1}^nx_i^{t+1}\right)\right\}\leq f\left(\frac{1}{n}\sum_{i=1}^nx_i^t\right)\\&\quad-\frac{\eta}{2\beta_i^{t}}E_{J}\left\{\left\|\nabla f\left(\frac{1}{n}\sum_{i=1}^{n}x_i^{t}\right)\right\|^{2}\right\}\\&\quad-\frac{\eta}{2\beta_i^{t}}E_{J}\left\{\left\|\frac{1}{n}\sum_{i=1}^{n}(g_{i}^{t}+d_{\tau}^{t}+v_{i}^{t})\right\|^{2}\right\}\\&\quad+\frac{\eta}{2\beta_i^{t}}E_{J}\left\{\left\|\nabla f\left(\frac{1}{n}\sum_{i=1}^{n}x_i^{t}\right)-\frac{1}{n}\sum_{i=1}^{n}(g_{i}^{t}+d_{\tau}^{t}+v_{i}^{t})\right\|^{2}\right\}\\&\quad+\frac{L\eta^{2}}{2(\beta_i^{t})^{2}}E_{J}\left\{\left\|\frac{1}{n}\sum_{i=1}^{n}(g_{i}^{t}+d_{\tau}^{t}+v_{i}^{t})\right\|^{2}\right\}\\&\quad+\frac{L\eta^{2}}{2(\beta_i^{t})^{2}}E_{J}\left\{\left\|\frac{1}{n}\sum_{i=1}^{n}\alpha_i^{T-t}u_{i}^{t}\right\|^{2}\right\}\\&\leq f\left(\frac{1}{n}\sum_{i=1}^{n}x_i^{t}\right)-\frac{\eta}{2\beta_i^{t}}E_{J}\left\{\left\|\nabla f\left(\frac{1}{n}\sum_{i=1}^{n}x_i^{t}\right)\right\|^{2}\right\}\\&\quad+\left(LO^{t}+E_{J}\left\{\frac{1}{n}\sum_{i=1}^{n}\|d_{\tau}^{t}+v_{i}^{t}\|^{2}\right\}\right)\\&\quad+\frac{L\eta^2}{2n}\frac{\left(\alpha_i^{T-t}\right)^2}{\left(\beta_i^t\right)^2}\frac{1}{n}\sum_{i=1}^nE_{J}\left\{\left\|u_{i}^{t}\right\|^{2}\right\}.\end{aligned}
		\end{aligned}
		\label{12}
	\end{equation}
	
	Taking expectation of both sides of (\ref{12}), rearranging, and summing from \(t = 0\) to \(t = T-1\), we have	
	\begin{equation}
		\begin{aligned}
			\frac{1}{T}&{\textstyle \sum_{t=0}^{T-1}}\frac{\eta}{2 \beta_i^t} E_{J}\left\{\left\|\nabla f\left(\frac{1}{n}\textstyle\sum_{i=1}^n x_i^t \right)\right\|\right\} \\
			&\le \frac{F_0}{T}+\frac{L\eta^2}{2n}\frac{\left(\alpha_i^{T-t}\right)^2}{\left(\beta_i^t\right)^2}\frac{1}{nT}{\textstyle \sum_{t=0}^{T-1}}{\textstyle \sum_{i=1}^{n}}E_{J}\left\{\left\|u_{i}^{t}\right\|^{2}\right\} \\
			&\quad +  \frac{L}{T} {\textstyle \sum_{t=0}^{T-1}}O^t + \frac{1}{T}{\textstyle \sum_{t=0}^{T-1}}E \left\{ \frac{1}{n}\sum_{i=1}^{n}\| d_\tau^{t}+  \upsilon_i^t \|^2 \right\}.
		\end{aligned}
		\label{13}
	\end{equation}
where $F_0 = f(x_0) - f^*$. 

	Let \( \bar{x}^{t} = \frac{1}{n}\textstyle\sum_{i=1}^n x_i^{t} \). Given \(\alpha_i^{t} = \sqrt{K} t^{\frac{1}{4} - \frac{p}{2}}\) with a positive constant \(K\), we deduce that \(\alpha_i^{t}\) is non-decreasing, implying \(\beta_i^{t} \leq K T^{\frac{1}{2}-p}\). Consequently, we derive the gradient-norm bound: ${\textstyle \sum_{t=0}^{T-1}\frac{\eta}{2 KT^{\frac{1}{2}-p}}{E}  \left \| \nabla{f\left ( \bar{x}^{t}  \right ) }  \right \|^{2} }\le{\textstyle \sum_{t=0}^{T-1}\frac{\eta}{2 \beta_i^t}{E}  \left \| \nabla{f\left ( \bar{x}^{t}  \right ) }  \right \|^{2} }$. It follows that 
	\begin{equation}
		\begin{aligned}
			\frac{1}{T} {\textstyle \sum_{t=0}^{T-1}}&  E \left\| \nabla f \left( \bar{x}^{t} \right) \right\|^2 \le \frac{K}{\eta} \left( \frac{2F_0}{T^{\frac{1}{2}+p}} + \frac{2L}{T^{\frac{1}{2}+p}} {\textstyle \sum_{t=0}^{T-1}} O^t \right. \\
			& \quad \left. + \frac{2}{T^{\frac{1}{2}+p}} {\textstyle \sum_{t=0}^{T-1}} \frac{1}{n} \sum_{i=1}^{n} E \left\{ \left\| d_\tau^t + \upsilon_i^t \right\|^2 \right\} \right. \\
			& \quad \left. + \frac{L \eta^2}{n T^{\frac{1}{2}+p}} {\textstyle \sum_{t=0}^{T-1}} \sum_{i=1}^{n} \frac{\left( \alpha_i^{T-t} \right)^2}{\left( \beta_i^t \right)^2} E \left\{ \left\| u_i^t \right\|^2 \right\} \right).
		\end{aligned}
		\label{14}
	\end{equation}

	Under Assumption~\ref{assumption:3}, let $\lambda = 1 - nU ^{-\kappa B}$ and $q =\lambda^{\frac{1}{{\kappa B + 1}}}$, where $U$ denotes the maximum number of out-neighbors for any node in any iteration. Then, there exists a constant $C < \frac{2\sqrt{d}U^{\kappa B}}{\lambda^{\frac{\kappa B+2}{\kappa B+1} }}$ such that for all nodes $i \in \{1, 2, \ldots, n\}$ and iterations $t\geq 0$, $\left \| z_{i}^t - \bar{x}^t \right \| $ is bounded by 
	\begin{equation} 
		\begin{aligned}
			\left \| z_{i}^t - \bar{x}^t \right \|  \leq Cq^t \left \| x_{i}^0 \right \|  + \eta_i ^{t} C {\textstyle\sum_{s=0}^{t} q^{t-s} \left \| \tilde g_{i}^{s} \right \| }. 
		\end{aligned}
		\label{15}
	\end{equation}
	
	It is derived from a minor modification of Theorem~1 in \cite{assran2020asynchronous}. 
	For (\ref{9}), we take the expectation as
	\begin{equation}
		\begin{aligned}
			{E}\{\tilde{g}_i^t\} = {E}\{g_i^t\} + {E}\{d_\tau^{t }\} + {E}\{\upsilon_i^t\} + {E}\{\alpha_i ^{T-t}u_i^{t }\}.
		\end{aligned}
		\label{l2}
	\end{equation}
	
	Thus, by applying (\ref{l2}), (\ref{15}) can be rewritten as
	\begin{equation}
		\begin{aligned}
			\left \| z_{i}^{t}-\bar{x} ^{t}    \right \| \le& Cq^{t} \left \| x_{i}^{0}  \right \| + C {\textstyle \sum_{s=0}^{t}}\eta_i ^{s}q^{t-s}\left \| g_{i}^{s}+d_\tau^{s}+\upsilon_i^s  \right \| \\
			&+ C {\textstyle \sum_{s=0}^{t}}\eta_i ^{s}q^{t-s}\left \| \alpha_i ^{T-s}u_{i}^{s}    \right \|.
		\end{aligned}
		\label{16}
	\end{equation}
	
	Taking the square of both sides, we obtain
	\begin{equation}
		\begin{aligned}
			\left \| z_{i}^{t}-\bar{x} ^{t}    \right \|^{2} &\leq 3C^{2}q^{2t} \left \| x_{i}^{0}  \right \|^{2} \\
			&+\frac{3C^{2} }{1-q}  {\textstyle \sum_{s=0}^{t}}\eta_i ^{2s}q^{t-s}\left \| g_{i}^{s}+d_\tau^{s}+\upsilon_i^s  \right \|^{2}\\
			&+\frac{3C^{2} }{1-q}  {\textstyle \sum_{s=0}^{t}}\eta_i ^{2s}q^{t-s}\left \| \alpha_i ^{T-s}u_{i}^{s}    \right \|^{2},
		\end{aligned} 
		\label{17}
	\end{equation}
	where we have used the Cauchy-Schwarz inequality. It follows 
	\begin{equation*}
		\begin{aligned}
			\left ( \textstyle\sum_{s=0}^{t} q^{t-s}  g_{i}^s \right )^2  &= \left ( \textstyle\sum_{s=0}^{t} q^\frac{t-s}{2} \left (  q^\frac{t-s}{2}  g_{i}^s \right )  \right ) ^2 \\
			&\leq \frac{1}{1-q} \textstyle\sum_{s=0}^{t} q^{t-s} \left ( g_{i}^s \right )^2.
		\end{aligned} 
	\end{equation*}
	
	Taking expectations of both sides of (\ref{17}), we have
	\begin{equation}
		\begin{aligned}
			{E}    \left \{  \left    \| z_{i}^{t}-\bar{x} ^{t}    \right \|^{2}  \right \}&\le 3C^{2}q^{2t} \left \| x_{i}^{0}  \right \|^{2}  \\
			&+\frac{3C^{2}\eta ^{2} }{1-q}  {\textstyle \sum_{s=0}^{t}}\frac{q^{t-s}}{\beta_i ^{2s} } {E}   \left \{ \left \| g_{i}^{s}+d_\tau^{s}+\upsilon_i^s  \right \|^{2} \right \}\\
			&+\frac{3C^{2}\eta ^{2} }{1-q}  {\textstyle \sum_{s=0}^{t}}\frac{q^{t-s}}{(\beta_i ^{s})^2 }{E}  \left \{  \left \| \alpha_i ^{T-s}u_{i}^{s}    \right \|^{2}  \right \}.
		\end{aligned}
		\label{18}
	\end{equation}
	For the second term in (\ref{18}), we have
	\begin{equation}
		\begin{aligned}
			{E}   &\left \{ \left \| g_{i}^{s}+d_\tau^{s}+\upsilon_i^s  \right \|^{2} \right \}  \le 3{E}   \left \{ \left \| g_{i}^{s}-\nabla f_{i}(z_{i}^{s})   \right \|^{2} \right \}  \\
			&+ 3{E}   \left \{ \left \| \nabla f_{i}(z_{i}^{s})   \right \|^{2} \right \} +3{E}   \left \{ \left \| d_\tau^{s}+\upsilon_i^s  \right \|^{2} \right \}.
		\end{aligned}
		\label{19}
	\end{equation}
    
	\vspace{20pt}

	For the first term in (\ref{19}), we have
	\begin{equation}
		\begin{aligned}
			{E}  &\left \{ \left \| g_{i}^{s}- \nabla f_{i}(z_{i}^{s})   \right \|^{2} \right \}\\
			&\le \frac{1}{\|B_i\|} {\textstyle \sum_{b \in B_i}} {E} \left \| \nabla f(z_{i}^{s},b )-\nabla f_{i}(z_{i}^{s})    \right \|^{2} \\
			&\le m^{2}.
		\end{aligned}
		\label{20}
	\end{equation}

	Next, for the second term in (\ref{19}), using Assumption~\ref{assumption:1} and Assumption~\ref{assumption:2}, we obtain
	\begin{equation}
		\begin{aligned}
			{E}&   \left \{ \left \| \nabla f_{i}(z_{i}^{s})   \right \|^{2} \right \}\le\left\|\nabla f_{i}\left(z_{i}^{s}\right)-\nabla f_{i}\left(\bar{x}^{s}\right)\right.\\ 
			&\left.\quad+\nabla f_{i}\left(\bar{x}^{s}\right)-\nabla f\left(\bar{x}^{s}\right)+\nabla f\left(\bar{x}^{s}\right)\right\|^{2}\\
			& \le 3\left\|\nabla f_{i}\left(z_{i}^{s}\right)-\nabla f_{i}\left(\bar{x}^{s}\right)\right\|^{2}+3\left\|\nabla f_{i}\left(\bar{x}^{s}\right)-\nabla f\left(\bar{x}^{s}\right)\right\|^{2}\\
			&\quad+3\left\|\nabla f\left(\bar{x}^{k}\right)\right\|^{2}\\
			&\le 3L^{2}{E}  \left \| z_{i}^{s}-\bar{x} ^{s}    \right \|^{2}+\!3a^{2}+3{E}\left \| \nabla f(\bar{x} ^{s}) \right \| ^{2}.
		\end{aligned}
		\label{21}
	\end{equation}
}
	
	Substituting (\ref{19}), (\ref{20}), and (\ref{21}) into (\ref{18}), we derive
	\begin{equation}
		\begin{aligned}
			{E}   & \left \{  \left    \| z_{i}^{t}-\bar{x} ^{t}    \right \|^{2}  \right \}\le 3C^{2}q^{2t} \left \| x_{i}^{0}  \right \|^{2} \\
			&\quad+\frac{9C^{2}\eta ^{2}\left ( m^{2}+3a^{2} \right )  }{1-q}  {\textstyle \sum_{s=0}^{t}}\frac{q^{t-s}}{(\beta_i ^{s})^2 }\\
			&\quad+  \frac{27C^{2}L^{2}\eta ^{2} }{1-q}  {\textstyle \sum_{s=0}^{t}}\frac{q^{t-s}}{(\beta_i ^{s})^2 } {E}    \left \{  \left    \| z_{i}^{s}-\bar{x} ^{s}    \right \|^{2}  \right \}\\
			&\quad+\frac{9C^{2}\eta ^{2} }{1-q}  {\textstyle \sum_{s=0}^{t}}\frac{q^{t-s}}{(\beta_i ^{s})^2 } {E}   \left \{ \left \| d_\tau^{s}+\upsilon_i^s   \right \|^{2} \right \} \\
			&\quad+\frac{27C^{2}\eta ^{2} }{1-q}  {\textstyle \sum_{s=0}^{t}}\frac{q^{t-s}}{(\beta_i ^{s})^2 } {E}\left \| \nabla f(\bar{x}^{s}) \right \| ^{2} \\
			&\quad   +\frac{3hC^{2}\eta ^{2}d\sigma _i ^{2} }{1-q}  {\textstyle \sum_{s=0}^{t}}\frac{q^{t-s}(\alpha_i^{(T-s)})^2}{(\beta_i ^{s})^2 },
		\end{aligned}
		\label{23}
	\end{equation}
	where $h=\frac{(1-\vartheta)}{(1+\vartheta)} + \frac{2}{(1+\vartheta)}\vartheta^{2\tau-1}$. 
	
	For simplicity, we assume that the parameters \( \alpha_i \) and \( \eta_i \) are identical for all nodes. Perform summation from $i=1$ to $i=n$, we have
	\begin{equation}
		\begin{aligned}
			\mathrm {O} ^{t}&= \frac{1}{n} {\textstyle \sum_{i=1}^{n}{E}    \left \{  \left    \| z_{i}^{t}-\bar{x} ^{t}    \right \|^{2}  \right \}} 	\\
			&\le 3C^{2}q^{2t} \left \| x_{i}^{0}  \right \|^{2} +\frac{9C^{2}\eta ^{2}\left ( m^{2}+3a^{2} \right )  }{1-q}  {\textstyle \! \sum_{s=0}^{t}}\frac{q^{t-s}}{(\beta_i ^{s})^2 }\\
			&\quad +  \frac{27C^{2}L^{2}\eta ^{2} }{1-q}  {\textstyle \sum_{s=0}^{t}}\frac{q^{t-s}}{(\beta_i ^{s})^2 } \mathrm {O} ^{s}\\
			&\quad+\frac{9C^{2}\eta ^{2} }{1-q}\frac{1}{n}  \sum_{i=1}^{n}  {\textstyle \sum_{s=0}^{t}}\frac{q^{t-s}}{(\beta_i ^{s})^2 } E \left \{ \left \| d_\tau^{s}+\upsilon_i^s   \right \|^{2} \right \} \\
			&\quad+\frac{27C^{2}\eta ^{2} }{1-q}  {\textstyle \sum_{s=0}^{t}}\frac{q^{t-s}}{(\beta_i ^{s})^2 } {E}\left \{ \left \| \nabla f(\bar{x} ^{s}) \right \| ^{2} \right \}  \\
			&\quad+\frac{3hC^{2}\eta ^{2}\frac{d}{n} {\textstyle \sum_{i=1}^{n}\sigma _i ^{2}}   }{1-q}  {\textstyle \sum_{s=0}^{t}}\frac{q^{t-s}(\alpha_i^{(T-s)})^2}{(\beta_i ^{s})^2 }.
		\end{aligned}
		\label{24}
	\end{equation}
	
	Summing from $t = 0$ to $t = T-1$ and reorganizing (\ref{24}), we further obtain
	\begin{equation}
		\begin{aligned}
			{\textstyle \sum_{t=0}^{T-1}}&\mathrm {O} ^{t}  \le   \frac{3C^{2}}{1-q^{2} }  \left \| x_{i}^{0}  \right \|^{2} \\
			&+\frac{9C^{2}\eta ^{2}\left ( m^{2}+3a^{2} \right )  }{1-q} {\textstyle \sum_{t=0}^{T-1}}{\textstyle \sum_{s=0}^{t}}\frac{q^{t-s}}{(\beta_i ^{s})^2 }   \\
			&+  \frac{27C^{2}L^{2}\eta ^{2} }{1-q}  {\textstyle \sum_{t=0}^{T-1}{\textstyle \sum_{s = 0}^{t}}\frac{q^{t-s}}{(\beta_i ^{s})^2 }}   \mathrm {O} ^{s}\\
			&+\frac{9C^{2}\eta ^{2} }{1-q}\frac{1}{n}  \sum_{i=1}^{n} {\textstyle \sum_{t=0}^{T-1}} {\textstyle \sum_{s=0}^{t}}\frac{q^{t-s}}{(\beta_i ^{s})^2 } {E}   \left \{ \left \| d_\tau^{s}+\upsilon_i^s   \right \|^{2} \right \} \\
			&+\frac{27C^{2}\eta ^{2} }{1-q}  {\textstyle \sum_{t=0}^{T-1}{\textstyle \sum_{s = 0}^{t}}\frac{q^{t-s}}{(\beta_i ^{s})^2 }} {E}\left \| \nabla f(\bar{x} ^{s}) \right \| ^{2}  \\
			&+\frac{3hC^{2}\eta ^{2}\frac{d}{n} {\textstyle \sum_{i = 1}^{n}\sigma _i ^{2}}   }{1-q}  {\textstyle \sum_{t=0}^{T-1}{\textstyle \sum_{s = 0}^{t}}\frac{q^{t-s}(\alpha_i^{(T-s)})^2}{(\beta_i ^{s})^2}}.
		\end{aligned}
		\label{25}
	\end{equation}
	
	Given \( \eta = \frac{K\sqrt{n}}{T^{p}} \) under \( \left(\frac{9}{\sqrt{n} }\right)^{\frac{4}{3-2p}} < T \), we have
	\begin{equation}
		\begin{aligned}
			{\textstyle \sum_{t=0}^{T-1}\mathrm {O} ^{t}}   \le&  \frac{3C^{2}}{1-q^{2} }  \left \| x_{i}^{0}  \right \|^{2} +\frac{C^{2}\left ( m^{2}+3a^{2} \right )  }{(1-q)^{2}}  \\
			&+  \frac{27C^{2}L^{2}\eta ^{2} }{(1-q)^{2}}  {\textstyle \sum_{t=0}^{T-1}{\frac{1}{(\beta_i ^{t})^2 }  }}   \mathrm {O} ^{t}\\
			&+\frac{C^{2}}{(1-q)^{2}} \frac{1}{n}  \sum_{i=1}^{n}{E}   \left \{ \left \| d_\tau^{s}+\upsilon_i^s   \right \|^{2} \right \} \\
			&+\frac{27C^{2}\eta ^{2} }{1-q}  {\textstyle \sum_{t=0}^{T-1}{\textstyle \sum_{s = 0}^{t}}\frac{q^{t-s}}{(\beta_i ^{s})^2 }} {E}\left \| \nabla f(\bar{x} ^{s}) \right \| ^{2}  \\
			&+\frac{3hC^{2}\eta ^{2}\frac{d}{n} {\textstyle \sum_{i = 1}^{n}\sigma _i ^{2}}   }{(1-q)^{2}}  {\textstyle \sum_{t=0}^{T-1}\frac{(\alpha_i^{(T-t)})^2}{(\beta_i ^{t})^2}}.
		\end{aligned}
		\label{26}
	\end{equation}
	
	Let \( D=1-\frac{81C^{2}L^{2} }{(1-q)^{2}} \frac{\eta^{2} }{(\beta_i ^{0})^{2}}\ge 0 \). Under \( T\geq \left(\frac{2K}{nL}\right)^{\frac{1}{1/2+p}} \), substituting (\ref{13}) into (\ref{26}), we obtain
	\begin{equation}
		\begin{aligned}
			&\frac{D}{T} \sum_{t=0}^{T-1} \mathrm{O}^{t}  \le \frac{1}{T} \left( \frac{3C^{2}}{1-q^{2}} \left\| x_{i}^{0} \right\|^{2} + \frac{54C^{2}}{(1-q)^{2}} \mathrm{F^{0}} \right. \\
			&\left. + \frac{9C^{2} \left( m^{2} + 3a^{2} \right)}{(1-q)^2} \right) + \frac{3hC^{2} \eta^{2} \frac{d}{n} \sum_{i=1}^{n} \sigma_i^{2}}{(1-q)^{2}DT} \sum_{t=0}^{T-1} \frac{(\alpha_i^{(T-t)})^2}{(\beta_i^{t})^2} \\
			& + \frac{27hLC^{2} \eta^{2} \frac{d}{n} \sum_{i=1}^{n} \sigma_i^{2}}{(1-q)^{2}nDT} \sum_{t=0}^{T-1} \frac{(\alpha_i^{(T-t)})^2}{(\beta_i^{t})^2} \\
			& + \frac{55C^{2}}{(1-q)^{2}DT} \frac{1}{n} \sum_{i=1}^{n} \sum_{t=0}^{T-1} \mathbb{E} \left\{ \left\| d_\tau^{t} + \upsilon_i^t \right\|^{2} \right\}.
		\end{aligned}
		\label{27}
	\end{equation}
	
	Given that \( \left(\frac{162nC^{2}L^{2}}{1-q^{2} }\right)^{\frac{2}{1+2p}} < T \), we derive the lower bound for $D$ as
	\begin{equation}
		\begin{aligned} 
			D=1-\frac{27c^{2}L^{2} }{(1-q)^{2}} (\frac{\eta^{2} }{(\beta_i ^{0})^{2}}+\frac{\eta }{(\beta_i ^{0})}) \ge  \frac{1}{2} .
		\end{aligned}
        \label{L2}
	\end{equation}
	
	By Lemma~\ref{lemma0}, we obtain $
        {\textstyle \sum_{t=0}^{T-1}{\frac{1}{\beta_i ^{t} }  }}
			\leq {\textstyle \sum_{t=0}^{T-1}{\frac{1}{ (\alpha_i^{t})^2 }  }}$. Substituting (\ref{L2}) into (\ref{27}), we rewrite (\ref{27}) as
	\begin{equation}
		\begin{aligned}
			&\frac{1}{nT}  {\textstyle \sum_{t=0}^{T-1}} {\textstyle \sum_{i=1}^{n}  {E} \left \{  \left    \| z_{i}^{t}-\bar{x} ^{t}    \right \|^{2}  \right \}}  \le \frac{A_1}{T}+ \left(\frac{2A_2h d\eta ^{2}}{Tn^2 }\cdot \right.\\
			&\quad \quad \left. {\textstyle\sum_{i=1}^n \frac{C^{2}G^{2}c_2^{2}\varsigma_i^{2}\ln \left(1/\delta_i\right)}{\epsilon_i^2}}{\textstyle \sum_{t=0}^{T-1}\frac{(\alpha_i^{(T-t)})^2}{(\beta_i ^{t})^2}}{\textstyle \sum_{t=0}^{T-1}\frac{1}{(\alpha_i ^{t})^2}}\right)\\ 
			&\quad \quad+\frac{A_3 }{T}  \frac{1}{n}  \sum_{i=1}^{n}{\textstyle  \textstyle\sum_{t=0}^{T-1}}  E \left\{ \| d_\tau^{t} \|^{2} + \| \upsilon_i^t \|^{2} \right\},
		\end{aligned}
		\label{29}
	\end{equation}
	where $A_1=\frac{6 C^2\|x_i^{0}\|+108 C^2 F_0+18 C^2\left(m^2+3 a^2\right)}{(1-q)^2}$
	, $A_2=\frac{6 C^2\left(n+9L\right)}{(1-q)^2}$ and $A_3=\frac{110 C^2}{(1-q)^2}. $
	Substituting (\ref{29}) into (\ref{14}), we have
	\begin{equation}
		\begin{aligned}
			\frac{1}{T}&{\textstyle \sum_{t=0}^{T-1}}{E}  \left \| \nabla{f\left ( \bar{x}^{t}  \right ) }  \right \|^{2} \le  \frac{ 2F_0+ 2LA_1}{\sqrt{n T}}\\
			&+ \frac{2+2LA_3}{ \sqrt{n T}}\frac{1}{n}\textstyle\sum_{i=1}^n{\textstyle \sum_{t=0}^{T-1}} E \left\{ \| d_\tau^{t} \|^{2} + \| \upsilon_i^t \|^{2} \right\}\\
			&+ \left( \frac{4A_2+1}{\sqrt{n T}} \frac{h dL\eta^2}{n^2} {\textstyle \sum_{i=1}^n \frac{C^{2}G^{2}c_2^{2}\varsigma_i^{2}\ln \left( 1/\delta_i \right)}{\epsilon_i^2}} \cdot \right. \\
			&\quad \quad \left. {\textstyle \sum_{t=0}^{T-1} \frac{(\alpha_i^{(T-t)})^2}{(\beta_i^{t})^2}} \cdot {\textstyle \sum_{t=0}^{T-1} \frac{1}{(\alpha_i^{t})^2}} \right).
		\end{aligned}
		\label{31}
	\end{equation}
	
	By Lemma~\ref{lemma1}, combining (\ref{29}) and (\ref{31}) under the condition \( T \geq 4L^2n \), we finally derive 
	\begin{equation}
		\begin{aligned}
			&\frac{1}{n T}{\textstyle \sum_{t=0}^{T-1}{\textstyle \sum_{i=1}^n E\left\{\left\|\nabla f\left(z_i^t\right)\right\|^2\right\}}}\\
			=&  \frac{1}{n T} {\textstyle \sum_{t=0}^{T-1} {\textstyle \sum_{i=1}^n E\left\{\left\|\nabla f\left(z_i^t\right)+\nabla f\left(\bar{x}^t\right)\right\|^2-\nabla f\left(\bar{x}^t\right) \|^2\right\}}} \\
			\le & 2 L^2 \frac{1}{n T}{\textstyle \sum_{t=0}^{T-1}{\textstyle \sum_{i=1}^n E\left\{\left\|z_i^t-\bar{x}^t\right\|^2\right\}}}\\
			&\quad \quad+2 \frac{1}{T}{\textstyle \sum_{t=0}^{T-1} E\left\{\left\|\nabla f\left(\bar{x}^t\right)\right\|^2\right\}} \\
			\le & 2 L^2 \frac{1}{T}{\textstyle \sum_{t=0}^{T-1} O^t}+2 \frac{1}{T} {\textstyle \sum_{t=0}^{T-1} E\left\{\left\|\nabla f\left(\bar{x}^t\right)\right\|^2\right\}} \\
			\le& \frac{4F_{0}+5LA_{1}}{\sqrt{nT}}+\frac{4hdL(5A_2+1)M}{\sqrt{nT} }+\frac{4 +5LA_{3}}{\sqrt{nT}}(\varrho +\upsilon),
		\end{aligned}
		\label{33}
	\end{equation}
	where \[M=\frac{\eta^2}{n^2} {\textstyle\sum_{i=1}^n \frac{G^{2}c_2^{2}\varsigma_i^{2}\ln \left(1/\delta_i\right)}{\epsilon_i^2}}{\textstyle \sum_{t=0}^{T-1}\frac{(\alpha_i^{(T-t)})^2}{(\beta_i ^{t})^2}}{\textstyle \sum_{t=0}^{T-1}\frac{1}{(\alpha_i ^{t})^2}},\] $\varrho +\upsilon=\frac{1}{n}\textstyle\sum_{i=1}^n{\textstyle \sum_{t=0}^{T-1}} {E}   \left \{ \left \| d_\tau^{t} \|^{2}+\| \upsilon_i^t   \right \|^{2} \right \},$ $h=\frac{(1-\vartheta)}{(1+\vartheta)} + \frac{2}{(1+\vartheta)}\vartheta^{2\tau-1}$  and 
	\begin{equation*}
		\begin{aligned}
			T &\geq \max\left\{ 4nL^2, \left(\frac{162nC^{2}L^{2}}{1-q^{2}}\right)^{\frac{2}{1+2p}},(nL^2)^{\frac{2}{1+2p}}, \right. \\
			&\quad\quad\quad\quad\quad \left. n^{\frac{2}{1+2p}},  \left(\frac{2K}{nL}\right)^{\frac{2}{1+2p}}, \left(\frac{9}{\sqrt{n}}\right)^{\frac{4}{3-2p}} \right\}.
		\end{aligned}
	\end{equation*}	
	
\section{Proof of Lemma1}\label{Appendix:A}
During the training process, the learning rate also affects the noise, as indicated by the ${\textstyle \sum_{t=0}^{T-1} {\frac{1}{ (\alpha_i^{T-t})^2 }  }}={\textstyle \sum_{t=0}^{T-1} {\frac{1}{ (\alpha_i^{t})^2 }  }}$, and the variance of the noise is given by
$ \frac{d}{n} \sum_{i=1}^{\eta} \sigma_i^2 \sum_{t=0}^{T} \frac{1}{(\alpha_i^t)^2}$.

By Algorithm~1, we have
\[ 
\beta_i^t = \begin{cases} \alpha_i^t \alpha_i^{(T-t)} & t \leq {\Xi}{T}  \\ \alpha_i^t \alpha_i^t & t > {\Xi}{T}  \end{cases}, 
\]
thus, we obtain
\begin{equation}
	\begin{aligned}
		{\textstyle \sum_{t=0}^{T-1}{\frac{1}{\beta_i ^{t} }  }}&= {\textstyle  \sum_{t=0}^{T-1}{\frac{1}{\alpha_i^{t} \alpha_i^{T-t} }  }}\\
		&\stackrel{(a)}{\leq} \sqrt{{\textstyle \sum_{t=0}^{T-1} {\frac{1}{ (\alpha_i^{t})^2 }  }}{\textstyle \sum_{t=0}^{T-1}{\frac{1}{ (\alpha_i^{T-t})^2 }  }}}\\
		&=\sqrt{{\textstyle \sum_{t=0}^{T-1}{\frac{1}{ (\alpha_i^{t})^2 }  }}{\textstyle \sum_{t=0}^{T-1}{\frac{1}{ (\alpha_i^{t})^2 }  }}}= {\textstyle \sum_{t=0}^{T-1}{\frac{1}{ (\alpha_i^{t})^2 }  }},
	\end{aligned}
	\label{28}
\end{equation}
where in (a) we used Cauchy-Schwarz inequality.

\section{Proof of Lemma2}\label{Appendix:B}
By Algorithm~2, we have $\tilde{g}_i^t = (1 - \vartheta) \bar{g}_i^{t} + \vartheta \tilde{g}_i^{t-1}$, where $\overline{g}_i^t = \hat{g}_i^t +\alpha_i^{T-t} \cdot u_i^t$. We define the discrepancy between the current gradient and the true gradient as \( \|d_{\tau}^{t}\| = \|(1-\vartheta)\hat{g}_{i}^{t} + \vartheta\hat{g}_{i}^{t-1} - g_{i}^{t}\|\). We further analyze \( \|d_{\tau}^{t}\|\):
\begin{equation}
	\begin{aligned}
		& E\left\{\|d_{\tau}^{t}\|^{2}\right\} = E\left\{\|(1-\vartheta)\hat{g}_{i}^{t} + \vartheta\hat{g}_{i}^{t-1} - g_{i}^{t}\|^{2}\right\} \\
		&= E\left\{ \|(1 - \vartheta)\hat{g}_{i}^{t}  +  \sum_{r=1}^{t-2}(1-\vartheta)\vartheta^{r}\hat{g}_{i}^{t-r}  +  \vartheta^{t-1}\hat{g}_{i}^{t-\tau+1} - \hat{g}_{i}^{t}\|^{2}  \right\} \\
		&= E\left\{\|(1-\vartheta)\hat{g}_{i}^{t} + \sum_{r=1}^{t-2}(1-\vartheta)\vartheta^{r}\hat{g}_{i}^{t} + \vartheta^{t-1}\hat{g}_{i}^{t} - \hat{g}_{i}^{t}\right. \\
		&\quad\quad \left.+ \sum_{r=1}^{t-2}(1-\vartheta)\vartheta^{r}(\hat{g}_{i}^{t-r} - \hat{g}_{i}^{t}) + \vartheta^{t-1}(\hat{g}_{i}^{t-\tau+1} - \hat{g}_{i}^{t})\|^{2}\right\} \\
		&= E\left\{\|\sum_{r=1}^{t-2}(1-\vartheta)\vartheta^{r}(\hat{g}_{i}^{t-r} - \hat{g}_{i}^{t}) + \vartheta^{t-1}(\hat{g}_{i}^{t-\tau+1} - \hat{g}_{i}^{t})\|^{2}\right\} \\
		&\leq h d_{\tau},
	\end{aligned}
	\label{32}
\end{equation}
where $h=\frac{(1-\vartheta)}{(1+\vartheta)}(1-\vartheta^{2(\tau-1)}) + \vartheta^{2(\tau-1)} = \frac{(1-\vartheta)}{(1+\vartheta)} + \frac{2}{(1+\vartheta)}\vartheta^{2\tau-1}$.

\section{Proof of Lemma3}\label{Appendix:C}
Since the noise variance within the interval remains the same, we have
\begin{equation}
	\begin{aligned}
		&E \left\{ \| \alpha_i^{T-t} u_i^t \|^2 \right\} = E \left\{ \| (1-\vartheta) u_i^t + \vartheta u_i^{t-1} \|^2 \right\} \\
		&= E \left\{ \left\| \sum_{r=0}^{t-2} (1-\vartheta) \vartheta^r \alpha_i^{T-t+r} u_i^{t-r}+ \vartheta^{t-1} \alpha_i^{T-(t-\tau+1)} u_i^{t-\tau+1} \right\} \right\} \\
		&\leq \left( \frac{(1-\vartheta)}{(1+\vartheta)} + \frac{2}{(1+\vartheta)} \vartheta^{2\tau-1} \right) E \left\{ \| \alpha_i^{T-t} u_i^t \|^2 \right\}.
	\end{aligned}
	\label{22}
\end{equation}

\section{Proof of Theorem1}\label{Appendix:D}
By making a slight modification to Lemma 3 in [9], we obtain
\begin{equation}
	\begin{aligned}
		\alpha _{M}(\lambda ) &= {\textstyle \sum_{t=0}^{T-1}} \alpha _{{M}_t }(\lambda )\le \frac{\lambda(\lambda+ 1)\varsigma_i^2}{2\sigma_i ^2}{\textstyle\sum_{t=1}^T \frac{1}{(\alpha_i^{T-t})^2}}\\
		&= \frac{\lambda(\lambda+ 1)\varsigma_i^2}{2\sigma_i ^2}{\textstyle\sum_{t=1}^T \frac{1}{(\alpha_i^{t})^2}}
	\end{aligned}
\end{equation}
Considering the fact that the sampling probability for each node $i$ is \( \varsigma_i=\frac{\|B_i\|}{\|D_i\|}\), the proof of Theorem 1 follows by extending the proof of Theorem 1 in [9]. 

\section{Proof of Corollary1}\label{Appendix:F}
For all \( a_{1}, a_{2}, a_{3} \in \mathbb{R} \), we derive the general form of \( \alpha_i^t \) as  
\[
(\alpha_i^t)^2=\frac{a_{1}^2}{(\lfloor\frac{t}{a_{3}}\rfloor+a_{2})^{2s}},
\]
We further analyze $\alpha_i^t$ and obtain
\begin{equation}
	\begin{aligned}
		(\alpha_i^t)^2=\frac{a_{1}^2}{(\lfloor\frac{t}{a_{3}}\rfloor+a_{2})^{2s}}\leq\frac{a_{3}a_{1}^2}{(t+a_{2})^{2s}}.
	\end{aligned}
\end{equation}
There exist constants $H_{1}$, $H_{2}$, and $H_{3}$, such that we have
\begin{equation}
	\begin{aligned}
		\sum_{t=0}^{T-1} (\alpha_i^t)^2 &\leq \sum_{t=0}^{T-1} \frac{a_{3}a_{1}^2}{(t+a_{2})^{2s}} \\
		&\leq a_{3}a_{1}^2 \left( 1 + \int_0^{T-1} \frac{\mathrm{d}t}{(t+a_{2})^{2s}} \right) \\
		&\leq \begin{cases} 
			H_{1} T^{1-2s} & 0 < s < \frac{1}{2} \\
			H_{2} \operatorname{log} T & s = \frac{1}{2} \\
			H_{3} & \frac{1}{2} < s < 1 
		\end{cases}
	\end{aligned}
\end{equation}

thus, by substituting into Theorem 2, we obtain Corollary 1.

 \bibliographystyle{IEEEtran}
\bibliography{IEEEabrv,cas-refs}

\vspace{60pt}
\begin{IEEEbiographynophoto}{Xiaoming Wu}
received the M.Eng. degree in computer science and technology from Shandong University, Jinan, China, in 2006, and the Ph.D. degree in Software Engineering from Shandong University of Science and Technology in 2017. Since 2006, he has been with the Shandong Computer Science Center, where he is currently a full professor. He also serves as the director of the Faculty of Computer Science and technology at Qilu University of Technology (Shandong Academy of Sciences), China. His research interests include cyber security, industrial Internet, data security, and privacy protection.
\end{IEEEbiographynophoto}

\vspace{-17pt}
\begin{IEEEbiographynophoto}{Teng Liu}
received the B.S. degree from the School of Mathematics and Statistics, Qilu University of Technology (Shandong Academy of Sciences), China, in 2022. He is currently pursuing a master's degree in computer science and technology at Shandong Computer Science Center, Qilu University of Technology (Shandong Academy of Sciences). His research interests include distributed learning and DP.
\end{IEEEbiographynophoto}

\vspace{-17pt}
\begin{IEEEbiographynophoto}{Xin Wang}
(Member, IEEE) received the B.E. degree in Electrical Engineering and Automation from China University of Mining and Technology, Xuzhou, China, in 2015, and the Ph.D. degree in Control Science and Engineering from Zhejiang University, Hangzhou, China, in 2020. He was a visiting scholar in the Department of Computer Science, Tokyo Institute of Technology, Yokohama, Japan, from 2018 to 2019. He is currently a Professor at the Shandong Computer Science Center, Qilu University of Technology (Shandong Academy of Sciences), Jinan, China.  His research interests include distributed artificial intelligence, federated learning, and data security.
\end{IEEEbiographynophoto}

\vspace{-17pt}
\begin{IEEEbiographynophoto}{Ming Yang}
(Member, IEEE) received the B.S. and M.S. degrees from the School of Information Science and Engineering, Shandong University, Jinan, China, in 2004 and 2007, respectively, and the Ph.D. degree from the School of Electronic Engineering, Beijing University of Posts and Telecommunications, Beijing, China, in 2010. He is currently a Professor with the Shandong Computer Science Center, Qilu University of Technology (Shandong Academy of Sciences), Jinan. His research interests include cloud computing security, big data security, and network security.
\end{IEEEbiographynophoto}

\vspace{-17pt}
\begin{IEEEbiographynophoto}{Jiguo Yu}
(Fellow, IEEE) received the Ph.D. degree from the School of Mathematics, Shandong University, in 2004. He was a Full Professor with the School of Computer Science, Qufu Normal University, Shandong, China, in 2007. He is currently a full professor with University of Electronic Science and Technology of China, Chengdu, Sichuan, China, and a joint professor at Qilu University of Technology, Jinan, Shandong, China. His main research interests include wireless network, privacy-aware computing, blockchain, intelligent IoT, AI and cyber security, distributed computing and graph theory. He is a Fellow of IEEE, a member of ACM and a senior member of the China Computer Federation (CCF).
\end{IEEEbiographynophoto}

\vfill

\end{document}